\documentclass{article}

\usepackage{microtype}
\usepackage{graphicx}
\usepackage{subfigure}
\usepackage{booktabs} 
\usepackage{amsmath, amssymb, amsthm}

\usepackage{hyperref}

\usepackage[accepted]{icml2020}

\icmltitlerunning{Infinite-Width Limits of Neural Classifiers}

\newcommand{\MM}{\mathcal{M}}
\newcommand{\BB}{\mathcal{B}}
\newcommand{\PP}{\mathcal{P}}
\newcommand{\LL}{\mathcal{L}}
\newcommand{\TT}{\mathcal{T}}
\newcommand{\WW}{\mathcal{W}}
\newcommand{\XX}{\mathcal{X}}
\newcommand{\EE}{\mathbb{E}\,}
\newcommand{\DD}{\mathcal{D}}
\newcommand{\RR}{\mathbb{R}}

\newcommand{\NN}{\mathcal{N}}

\newcommand{\ww}{\mathbf{w}}

\newcommand{\Lip}{\mathrm{Lip}}

\newcommand{\xx}{\mathbf{x}}
\renewcommand{\aa}{\mathbf{a}}
\newcommand{\Var}{\mathbb{V}\mathrm{ar}\,}
\newcommand{\const}{\mathrm{const}}

\newtheorem{thm}{Theorem}

\newtheorem{lemma}{Lemma}
\newtheorem{corollary}{Corollary}

\begin{document}

\twocolumn[
\icmltitle{Towards a General Theory of Infinite-Width Limits of Neural Classifiers}



\icmlsetsymbol{equal}{*}

\begin{icmlauthorlist}
\icmlauthor{Eugene A. Golikov}{mipt}
\end{icmlauthorlist}

\icmlaffiliation{mipt}{Neural Networks and Deep Learning lab., Moscow Institute of Physics and Technology, Moscow, Russia}

\icmlcorrespondingauthor{Eugene A. Golikov}{golikov.ea@mipt.ru}

\icmlkeywords{mean-field limit, neural tangent kernel}

\vskip 0.3in
]



\printAffiliationsAndNotice{}  

\begin{abstract}

    Obtaining theoretical guarantees for neural networks training appears to be a hard problem in a general case. 
    Recent research has been focused on studying this problem in the limit of infinite width and two different theories have been developed: a mean-field (MF) and a constant kernel (NTK) limit theories. 
    We propose a general framework that provides a link between these seemingly distinct theories. 
    Our framework out of the box gives rise to a discrete-time MF limit which was not previously explored in the literature. 
    We prove a convergence theorem for it, and show that it provides a more reasonable approximation for finite-width nets compared to the NTK limit if learning rates are not very small. 
    Also, our framework suggests a limit model that coincides neither with the MF limit nor with the NTK one.
    We show that for networks with more than two hidden layers RMSProp training has a non-trivial discrete-time MF limit but GD training does not have one. 
    Overall, our framework demonstrates that both MF and NTK limits have considerable limitations in approximating finite-sized neural nets, indicating the need for designing more accurate infinite-width approximations for them.
\end{abstract}

\section{Introduction}
\label{sec:intro}

Despite neural networks' great success in solving a variety of problems, theoretical guarantees for their training are scarce and far from being practical.
It turns out that neural models of finite size are very complex objects to study since they usually induce a non-convex loss landscape. This makes it highly non-trivial to obtain any theoretical guarantees for the gradient descent training.

However theoretical analysis becomes tractable in the limit of infinite width.
In particular, \cite{jacot2018ntk} showed that if weights are parameterized in a certain way then the continuous-time gradient descent on neural network parameters converges to a solution of a kernel method.
The corresponding kernel is called a neural tangent kernel (NTK).

Another line of work studies a mean-field (MF) limit of the training dynamics of neural nets with a single hidden layer \cite{mei2018mf,mei2019mf,rotskoff2019mf,sirignano2018lln,chizat2018mf,yarotsky2018mf}.
In these works a neural net output is scaled differently compared to the work on NTK.

In our work we address several questions arising in this context:
\begin{enumerate}
    \item Which of these two limits appears to be a more reasonable approximation for a finite-width network?
    \item Do the two above-mentioned limits cover all possible limit models for neural networks?
    \item Is it possible to construct a non-trivial mean-field limit for a multi-layer network?
\end{enumerate}

The paper is organized as follows.
In Section~\ref{sec:related} we provide a brief review of the relevant studies.
In Section~\ref{sec:1hid_gd} we consider hyperparameter scalings that lead to non-trivial infinite-width limits for neural nets with a single hidden layer.
Our analysis clearly shows that MF and NTK limits are not the only possible ones. 
Also, our analysis suggests a discrete-time MF limit which appears to be a more reasonable approximation for a finite-sized neural network than the NTK limit if learning rates are not very small.
We stress the difference between this discrete-time MF limit and a continuous-time one described in previous works and prove a convergence theorem for it.
In Section~\ref{sec:Hhid_gd} we show that when a neural net has at least three hidden layers a discrete-time MF limit becomes vanishing.
Nevertheless, training a network with RMSProp instead of a plain gradient descent leads to a non-trivial discrete-time MF limit for any number of layers.


\section{Related work}
\label{sec:related}

\paragraph{NTK limit.}
In their pioneering work \citet{jacot2018ntk} considered a multi-layer feed-forward network parameterized as follows:
\begin{equation}
    f(\xx; W_{1:L}) =
    d_{L-1}^{-1/2} W_L \phi(d_{L-2}^{-1/2} W_{L-1} \ldots \phi(d_0^{-1/2} W_1 \xx)),
    \label{eq:ntk_model}
\end{equation}
where $\xx \in \RR^{d_0}$, $d_i$ is a size of the $i$-th layer and $W_l \in \RR^{d_l \times d_{l-1}}$.
The weights are initialized as $W_{l,ij}^{(0)} \sim \NN(0,1)$.

\citet{jacot2018ntk} have shown that training this model with a continuous-time gradient descent is equivalent to performing a kernel gradient descent for some specific kernel; they called this kernel a neural tangent kernel (NTK).
This kernel is generally stochastic and evolves with time, however, as they prove, it converges to a steady-state deterministic kernel as $d_{1:L-1} \to \infty$.

\citet{lee2019wide} have shown that the training dynamics of the network (\ref{eq:ntk_model}) stays close to the training dynamics of its linearized version in the limit of infinite width; the linearization is performed with respect to weights.
They also show that this statement holds for the discrete-time gradient descent as long as the learning rates are sufficiently small.

\citet{arora2019exact} provide a way to effectively compute the NTK for convolutional neural networks.
They found that a kernel method with the NTK still performs worse than the corresponding finite-width CNN.
At the same time, as was noted by \citet{lee2019wide}, the training dynamics in the NTK limit is effectively linear. 
\citet{bai2019beyond} artificially created a situation where a linearized dynamics was not able to track the training dynamics in the limit of infinite width.
These two works show that the NTK limit is not perfect in the sense that it can be far from a realistic finite-size neural net.

\paragraph{Mean-field limit.}
There is a line of works \cite{mei2018mf,mei2019mf,rotskoff2019mf,sirignano2018lln,chizat2018mf,yarotsky2018mf} that consider a two-layer neural net of width $d$ in a mean-field limit:
\begin{equation}
    f(\xx; \aa, W) =
    d^{-1} \aa^T \phi(W^T \xx) = 
    d^{-1} \sum_{r=1}^d a_r \phi(\ww_r^T \xx),
    \label{eq:mf_model}
\end{equation}
where $\xx \in \RR^{d_0}$; the weights are initialized independently on the width $d$ and $d$ goes to infinity.
Note the difference in scaling the output function between (\ref{eq:mf_model}) and (\ref{eq:ntk_model}) for $L=2$.
In the present case any weight configuration can be expressed as a point measure in $(a,\ww)$-space $\RR^{d_0+1}$:
\begin{equation*}
    \mu[\aa, W] = d^{-1} \sum_{r=1}^d \delta_{a_r} \otimes \delta_{\ww_r}.
\end{equation*}
A neural network is then expressed as an integral over the measure:
\begin{equation}
    f(\xx; \aa, W) =
    \int a \phi(\ww^T \xx) \, \mu[\aa, W](da, d\ww).
    \label{eq:f_measure_int}
\end{equation}

The above-mentioned works show that when learning rates are appropriately scaled width $d$, a gradient descent dynamics turns into a continuous-time dynamics for the measure $\mu$ in $(a,\ww)$-space driven by a certain PDE as $d$ goes to infinity.
This evolution in the weight space also drives the evolution of the model $f$ (see~ (\ref{eq:f_measure_int})).

Note that those works that study a limit behavior of the discrete-time gradient descent \cite{sirignano2018lln,mei2018mf,mei2019mf} require the number of training steps to grow with $d$ since they prove convergence to a continuous-time dynamics.
In contrast, in our work we find a similar mean-field-type limit that converges to a discrete-time limit dynamics.

There are several attempts to extend the mean-field analysis to multi-layer nets \cite{sirignano2019deep_mf,nguyen2019deep_mf,fang2019deep_mf}.
However this appears to be highly non-trivial to formulate a measure evolution PDE similar to a single-hidden-layer case (see the discussion of difficulties in Section 3.3 of \citet{sirignano2019deep_mf}).
In particular, \citet{sirignano2019deep_mf} rigorously constructed an iterated mean-field limit for a two-hidden-layer case.
In contrast, the construction of \citet{nguyen2019deep_mf} applies to any number of layers while not being mathematically rigorous.
\citet{fang2019deep_mf} claim to find a way to represent a deep network as a sequence of integrals over a system of probability measures.
Given this, the loss becomes convex as a function of this system of measures.
However they do not consider any training process.

It also has to be noted that \citet{nguyen2019deep_mf} applied a weight initialization with a non-zero mean for their experiments with scaling multi-layer nets.
As we show in Section~\ref{sec:Hhid_gd}, if the number of hidden layers is more than two and initialization has zero mean (which is common in deep learning), a mean-field limit becomes trivial.



\section{Training a one hidden layer net with gradient descent}
\label{sec:1hid_gd}

Here we consider a simple case of networks with a single hidden layer of width $d$ trained with GD.
Our goal is to deduce how one should scale its training hyperparameters (learning rates and initialization variances) in order to converge to a non-trivial limit model evolution as $d \to \infty$.
We say that a limit model evolution is non-trivial if the model neither vanishes, nor diverges, and varies over the optimization process.
We formalize this notion later in the text.
We investigate certain classes of hyperparameter scalings that lead to non-trivial limit models.
We find both existing MF and NTK scalings, as well as a different class of scalings that lead to a model that does not coincide with either MF or NTK limit.
Finally, we discuss the ability of limit models to approximate finite-width nets.

Consider a one hidden layer net of width $d$:
\begin{equation*}
    f(\xx; \aa, W) =
    \aa^T \phi(W^T \xx) = 
    \sum_{r=1}^d a_r \phi(\ww_r^T \xx),
\end{equation*}
where $\xx \in \RR^{d_0}$, $W = [\ww_1, \ldots, \ww_d] \in \RR^{d_0 \times d}$, and $\aa = [a_1, \ldots, a_d]^T \in \RR^d$.
The nonlinearity $\phi(z) = [z]_+ - \alpha [-z]_+$ for $\alpha > 0$ is considered to be the leaky ReLU and applied element-wise.
We consider a loss function $\ell(y,z)$ that is continuosly-differentiable with respect to the second argument.
We also assume $\partial \ell(y,z) / \partial z$ to be positive continuous and monotonic $\forall y$.
The guiding example is the standard cross-entropy loss.
The data distribution loss is defined as $\LL(\aa, W) = \EE_{\xx,y \in D_{train}} \ell(y, f(\xx; \aa,W))$, where $D_{train}$ is a train dataset sampled from the data distribution $\DD$.

Weights are initialized with isotropic gaussians with zero means: $\ww_r^{(0)} \sim \NN(0, \sigma_w^2 I)$, $a_r^{(0)} \sim \NN(0, \sigma_a^2)$ $\forall r = 1\ldots d$.
The evolution of weights is driven by the gradient descent dynamics:
\begin{equation*}
    \Delta\theta_r^{(k)} = \theta_r^{(k+1)} - \theta_r^{(k)} = - \eta_\theta \frac{\partial \LL(\aa^{(k)}, W^{(k)})}{\partial \theta_r},
\end{equation*}
where $\theta$ is either $a$ or $\ww$.

Initialization variances, $\sigma_a^2$ and $\sigma_w^2$, generally depend on $d$: e.g. $\sigma_a^2 \propto d^{-1}$ for He initialization \cite{he2015init}.
This fact complicates the study of the limit $d\to\infty$.
To work around this, we rescale our hyperparameters:
$$
    \hat a_r^{(k)} = \frac{a_r^{(k)}}{\sigma_a}, \quad
    \hat\ww_r^{(k)} = \frac{\ww_r^{(k)}}{\sigma_w}, \quad
    \hat \eta_a = \frac{\eta_a}{\sigma_a^2}, \quad
    \hat \eta_w = \frac{\eta_w}{\sigma_w^2}.
$$
The GD dynamics preserves its form:
\begin{equation*}
    \Delta\hat\theta_r^{(k)} = -\hat\eta_\theta \frac{\partial \LL(W^{(k)}, \aa^{(k)})}{\partial \hat\theta_r}.
\end{equation*}
At the same time, scaled initial conditions do not depend on $d$ anymore: $\hat a_r^{(0)} \sim \NN(0, 1)$, $\hat\ww_r^{(0)} \sim \NN(0, I)$ $\forall r = 1\ldots d$.

By expanding gradients we get the following: 
\begin{equation}
    \Delta\hat a_r^{(k)} = 
    - \hat\eta_a \sigma_a \sigma_w \EE_{\xx,y} \nabla_f^{(k)} \ell \; \phi(\hat\ww_r^{(k),T} \xx),
\end{equation}
\begin{equation}
    \Delta\hat\ww_r^{(k)} = 
    - \hat\eta_w \sigma_a \sigma_w \EE_{\xx,y} \nabla_f^{(k)} \ell \; \hat a_r^{(k)} \phi'(\hat\ww_r^{(k),T} \xx) \xx,
\end{equation}
\begin{equation}
    \hat a_r^{(0)} \sim \NN(0, 1), \quad \hat\ww_r^{(0)} \sim \NN(0, I) \quad \text{for all $r=1\ldots d$,}
\end{equation}
where we have denoted $f_d^{(k)}(\xx) = \sigma_a \sum_{r=1}^d \hat a^{(k)}_r \phi(\sigma_w \hat\ww_r^{(k),T} \xx)$ and $ \nabla_f^{(k)} \ell = \left. \frac{\partial \ell(y,z)}{\partial z} \right|_{z=f_d^{(k)}(\xx)}$.
We have also used the fact that $\phi(\sigma z) = \sigma \phi(z)$ for $\phi$ being the leaky ReLU.
We shall omit $\xx, y$ in the expectation from now on.

Denote $\sigma = \sigma_a \sigma_w$.
Assume hyperparameters that drive the dynamics are scaled with $d$:
\begin{equation*}
    \sigma \propto d^{q_\sigma}, \quad
    \hat\eta_a \propto d^{\tilde q_a}, \quad
    \hat\eta_w \propto d^{\tilde q_w}.
\end{equation*}
We call a set of exponents $(q_\sigma, \tilde q_a, \tilde q_w)$ "a scaling".
Every scaling define a limit model $f_\infty^{(k)}(\xx) = \lim_{d \to \infty} f_d^{(k)}(\xx)$.
We want this limit to be non-divergent, non-vanishing and not equal to the initialization $f_d^{(0)}$ for any $k \geq 1$.
We call such scalings and corresponding limit models non-trivial.

\subsection{Analyzing non-triviality}

We start with introducing weight increments:
\begin{equation*}
    \delta\hat a_r^{(k)} = \hat a_r^{(k)} - \hat a_r^{(0)}, \quad
    \delta\hat\ww_r^{(k)} = \hat\ww_r^{(k)} - \hat\ww_r^{(0)}.
\end{equation*}
Since our dynamics is symmetric with respect to permutation of indices $r$, we can assume the following:
\begin{equation}
    |\delta\hat a_r^{(k)}| \propto d^{q_a^{(k)}}, \quad
    \|\delta\hat\ww_r^{(k)}\| \propto d^{q_w^{(k)}}.
    \label{eq:increments_power_law}
\end{equation}
Our intuition here is that $|\delta\hat a_r^{(k)}| \sim d^{-1/2}$ for the NTK scaling, while $|\delta\hat a_r^{(k)}| \sim d^{0}$ for the MF scaling. 
We validate the assumption above numerically for some of the scalings in SM~C.
We proceed with decomposing the model:
\begin{multline}
    f_d^{(k)}(\xx) = 
    \sigma \sum_{r=1}^d (\hat a_r^{(0)} + \delta\hat a_r^{(k)}) \phi'(\ldots) (\hat\ww_r^{(0)} + \delta\hat\ww_r^{(k)})^T \xx \\=
    f_{d,\emptyset}^{(k)}(\xx) + f_{d,a}^{(k)}(\xx) + f_{d,w}^{(k)}(\xx) + f_{d,aw}^{(k)}(\xx), 
    \label{eq:f_decomposition_1hid}
\end{multline}
where we define the decomposition terms as:
\begin{equation*}
    f_{d,\emptyset}^{(k)}(\xx) =
    \sigma \sum_{r=1}^d \hat a_r^{(0)} \phi'(\ldots) \hat\ww_r^{(0),T} \xx,
\end{equation*}
\begin{equation*}
    f_{d,a}^{(k)}(\xx) =
    \sigma \sum_{r=1}^d \delta\hat a_r^{(k)} \phi'(\ldots) \hat\ww_r^{(0),T} \xx,
\end{equation*}
\begin{equation*}
    f_{d,w}^{(k)}(\xx) =
    \sigma \sum_{r=1}^d \hat a_r^{(0)} \phi'(\ldots) \delta\hat\ww_r^{(k),T} \xx,
\end{equation*}
\begin{equation*}
    f_{d,aw}^{(k)}(\xx) =
    \sigma \sum_{r=1}^d \delta\hat a_r^{(k)} \phi'(\ldots) \delta\hat\ww_r^{(k),T} \xx.
\end{equation*}
Here $\phi'(\ldots)$ is a shorthand for $\phi'((\hat\ww_r^{(0)} + \delta\hat\ww_r^{(k)})^T \xx)$.

Since all of the terms inside the sums are presumably power-laws of $d$ (or at least do not grow or vanish with $d$), it is natural to assume power-laws for the decomposition terms also (see SM~C for empiricial validation):
\begin{equation*}
    f_{d,\emptyset}^{(k)}(\xx) \propto d^{q_{f,\emptyset}^{(k)}}, \;
    f_{d,a/w}^{(k)}(\xx) \propto d^{q_{f,a/w}^{(k)}}, \;
    f_{d,aw}^{(k)}(\xx) \propto d^{q_{f,aw}^{(k)}}
\end{equation*}
Here and later we will write "$a/w$" meaning "$a$ or $w$".

The introduced assumptions allow us to formulate the non-triviality condition in terms of the power-law exponents:
\begin{equation}
    \max(q_{f,\emptyset}^{(k)}, q_{f,a}^{(k)}, q_{f,w}^{(k)}, q_{f,aw}^{(k)}) = 0 \; \forall k \geq 1;
    \label{eq:non-trivialness_condition1}
\end{equation}
\begin{equation}
    \max(q_{f,a}^{(k)}, q_{f,w}^{(k)}, q_{f,aw}^{(k)}) = 0 \; \text{or} \; q_{f,\emptyset}^{(k)} = 0 \; \text{and} \; q_w^{(k)} \geq 0.
    \label{eq:non-trivialness_condition2}
\end{equation}
The first condition ensures that $\lim_{d \to \infty} f_d^{(k)}$ is finite and not uniformly zero, while the second one ensures that this limit does not coincides with the initialization (hence the learning dynamics does not get stuck as $d \to \infty$).
In particular, the second condition requires either one of $f_{d,a}^{(k)}$, $f_{d,w}^{(k)}$, or $f_{d,aw}^{(k)}$ to contribute substantially to $f_d^{(k)}$ for large $d$, or, if the leading term is $f_{d,\emptyset}$, it requires $\lim_{d\to\infty} f_{d,\emptyset}^{(k)}$ not to coincide with $\lim_{d\to\infty} f_d^{(0)}$ (because $\phi'((\hat\ww^{(0)}+\delta\hat\ww^{(k)})^T \xx) \nrightarrow \phi'(\hat\ww^{(0),T} \xx)$ as $d\to\infty$ if $q_w^{(k)} \geq 0$).

In order to test Conditions~(\ref{eq:non-trivialness_condition1}) and (\ref{eq:non-trivialness_condition2}), we have to relate the introduced $q$-exponents with the scaling $(q_\sigma, \tilde q_a, \tilde q_w)$.
From the definition of decomposition~(\ref{eq:f_decomposition_1hid}) terms we get:
\begin{equation*}
    q_{f,\emptyset}^{(k)} = q_\sigma + \varkappa_\emptyset^{(k)}, \quad
    q_{f,a/w}^{(k)} = q_{a/w}^{(k)} + q_\sigma + \varkappa_{a/w}^{(k)},
\end{equation*}
\begin{equation}
    q_{f,aw}^{(k)} = q_a^{(k)} + q_w^{(k)} + q_\sigma + \varkappa_{aw}^{(k)},
    \label{eq:q_terms_def_1hid}
\end{equation}
where all $\varkappa \in \{1/2,1\}$.
We now use $q_{f,a}^{(k)}$ to illustrate where these equations come from.
We have:
\begin{multline*}
    \EE_{\hat\aa^{(0)},\hat W^{(0)}} f_{d,a}^{(k)}(\xx) =
    \sigma d \EE \delta\hat a^{(k)} \phi'(\ldots) \hat\ww^{(0),T} \xx =\\=
    \sigma d^{1+q_a^{(k)}} \EE \frac{\delta\hat a^{(k)}}{d^{q_a^{(k)}}} \phi'(\ldots) \hat\ww^{(0),T} \xx,
\end{multline*}
since all terms of the sum have the same expectation.
Hence if the last expectation is non-zero in the limit of $d \to \infty$, then we have $\EE f_{d,a}^{(k)}(\xx) \propto \sigma d^{1+q_a^{(k)}}$ and consequently $q_{f,a}^{(k)}(\xx) = q_a^{(k)} + q_\sigma + 1$; so, $\varkappa_a^{(k)} = 1$.
However, if it is zero in the limit of $d \to \infty$, then we have to reason about the variance.
We have $\Var f_{d,a}^{(k)}(\xx) \propto \sigma^2 d^{1+2 q_a^{(k)}}$ if all terms of the sum appear to be independent in the limit of $d \to \infty$, or $\Var f_{d,a}^{(k)}(\xx) \propto \sigma^2 d^{2+2 q_a^{(k)}}$ if they are perfectly correlated.
Hence $q_{f,a}^{(k)}(\xx) = q_a^{(k)} + q_\sigma + \varkappa_a^{(k)}$, where $\varkappa_a^{(k)} \in \{1/2,1\}$.
Generally, all $\varkappa$-terms can be defined if $q_a^{(k)}$ and $q_w^{(k)}$ are known.

We now relate $q_{a/w}^{(k)}$ with the scaling $(q_\sigma, \tilde q_a, \tilde q_w)$.
First, we rewrite the dynamics in terms of weight increments:
\begin{equation*}
    \Delta\delta\hat a_r^{(k)} = 
    - \hat\eta_a \sigma \EE \nabla_f^{(k)} \ell \; \phi((\hat\ww_r^{(0)} + \delta\hat\ww_r^{(k)})^T \xx),
\end{equation*}
\begin{equation}
    \Delta\delta\hat\ww_r^{(k)} = 
    - \hat\eta_w \sigma \EE \nabla_f^{(k)} \ell \; (\hat a_r^{(0)} + \delta\hat a_r^{(k)}) \phi'(\ldots) \xx,
    \label{eq:1hid_dynamics}
\end{equation}
\begin{equation*}
    \delta\hat a_r^{(0)} = 0, \; \delta\hat\ww_r^{(0)} = 0, \;
    \hat a_r^{(0)} \sim \NN(0, 1), \; \hat\ww_r^{(0)} \sim \NN(0, I).
\end{equation*}
Recall that we are looking for scalings that lead to a non-divergent limit model $f_\infty^{(k)}$; hence $f_d^{(k)}$ should not grow with $d$.
Then, since $\nabla_f^{(k)} \ell$ is strictly positive continuous and monotonic $\forall y$, we have $|\nabla_f^{(k)} \ell|$ bounded away from zero as a function of $d$.
Also, since $\hat a_r^{(0)} \propto 1$ and $\|\hat\ww_r^{(0)}\| \propto 1$, from the dynamics equations (\ref{eq:1hid_dynamics}) we get:
\begin{equation}
    q_a^{(1)} = \tilde q_a + q_\sigma, \quad
    q_w^{(1)} = \tilde q_w + q_\sigma,
    \label{eq:q1_1hid}
\end{equation}
\begin{equation*}
    q_a^{(k+1)} = \max(q_a^{(k)}, \tilde q_a + q_\sigma + \max(0, q_w^{(k)})),
\end{equation*}
\begin{equation*}
    q_w^{(k+1)} = \max(q_w^{(k)}, \tilde q_w + q_\sigma + \max(0, q_a^{(k)})).
\end{equation*}
The last two equations can be rewritten as:
\begin{equation}
    q_{a/w}^{(k+1)} = \max(q_{a/w}^{(k)}, q_{a/w}^{(1)} + \max(0, q_{w/a}^{(k)})).
    \label{eq:qk_1hid}
\end{equation}
Here we have used the following heuristic rules:
\begin{equation*}
    u \propto d^{q_u}, \; v \propto d^{q_v} 
    \Rightarrow
    u v \propto d^{q_u+q_v}, \; u+v \propto d^{\max(q_u,q_v)}.
\end{equation*}
Although these rules are not mathematically correct, we empirically validated the exponents predicted by equations (\ref{eq:q1_1hid}) and (\ref{eq:qk_1hid}): see SM~C.

The $\varkappa$-terms together with equations (\ref{eq:non-trivialness_condition1}), (\ref{eq:non-trivialness_condition2}), (\ref{eq:q_terms_def_1hid}), (\ref{eq:q1_1hid}), and (\ref{eq:qk_1hid}) define a set of sufficient conditions for a scaling $(q_\sigma, \tilde q_a, \tilde q_w)$ to define a non-trivial limit model.
In the next section, we derive several solution classes for this system of equations.
These classes contain both MF and NTK scalings, as well as a family of scalings that lead to a limit model that coincides with neither MF, nor NTK limits.

\subsection{Non-trivial limits}

Although deriving $\varkappa$-terms appears to be quite complicated generally, we derive them for several special cases.



Consider the case of $q_a^{(1)} < 0$ and $q_w^{(1)} < 0$.
Equations (\ref{eq:qk_1hid}) imply $q_a^{(k)} = q_a^{(1)}$ and $q_w^{(k)} = q_w^{(1)}$ $\forall k \geq 1$ then.
We also conclude (see SM~D) that $\varkappa_\emptyset^{(k)} = \varkappa_{aw}^{(k)} = 1/2$ and $\varkappa_{a/w}^{(k)} = 1$ in this case.

Conditions (\ref{eq:non-trivialness_condition1}) and (\ref{eq:non-trivialness_condition2}) then imply $q_\sigma \leq -1/2$ and $q_{a/w}^{(1)} \leq -1-q_\sigma$ with an equality for at least one of $q_a^{(1)}$ or $q_w^{(1)}$.
Because of the latter, and since in our case we have to have $\max(q_a^{(1)}, q_w^{(1)}) < 0$, we get a constraint $q_\sigma > -1$.
Note also that in this case $q_{f,aw}^{(k)} < 0$.

Hence by taking $q_\sigma \in (-1,-1/2]$ and $\tilde q_{a/w} = q_{a/w}^{(1)} - q_\sigma \leq -1-2q_\sigma$ with at least one inequality being an equality, we define a non-trivial scaling.
As a particular example of this case consider $q_\sigma = q_a^{(1)} = q_w^{(1)} = -1/2$.
It follows than from (\ref{eq:q1_1hid}) that $\tilde q_a = \tilde q_w = 0$.
If we take $\hat\eta_a = \hat\eta_w = \hat\eta$ and $\sigma = \sigma^* d^{-1/2}$ then we get the following relations:
\begin{equation*}
    f_d^{(k)}(\xx) = \sum_{r=1}^d \sigma^* d^{-1/2} \hat a^{(k)}_r \phi(\hat\ww_r^{(k),T} \xx),
\end{equation*}
\begin{equation*}
    \Delta\hat\theta_r^{(k)} = - \hat\eta \frac{\partial \LL(W^{(k)}, \aa^{(k)})}{\partial \hat\theta_r}, \quad \theta \in \{a,\ww\},
\end{equation*}
\begin{equation*}
    \hat a_r^{(0)} \sim \NN(0, 1), \quad \hat\ww_r^{(0)} \sim \NN(0, I) \quad \text{for all $r=1\ldots d$.}
\end{equation*}
This system exactly corresponds to the one used in the NTK theory \cite{jacot2018ntk} (see also eq.~(\ref{eq:ntk_model})).

Following \cite{jacot2018ntk,lee2019wide}, we call a neural tangent kernel (NTK) the following function:
\begin{eqnarray*}
    \Theta_d^{(k)}(\xx,\xx') 
    =
    \sum_{r=1}^d \Biggl(
        \frac{\partial f^{(k)}(\xx)}{\partial \hat a_r} \frac{\partial f^{(k)}(\xx')}{\partial \hat a_r} +\\+
        \frac{\partial f^{(k)}(\xx)}{\partial \hat\ww_r} \left(\frac{\partial f^{(k)}(\xx')}{\partial \hat\ww_r}\right)^T
    \Biggr)
    =\\=
    \sigma^2 \sum_{r=1}^d \bigl(\phi(\hat\ww_r^{(k),T} \xx) \phi(\hat\ww_r^{(k),T} \xx') +\\+ \phi'(\hat\ww_r^{(k),T} \xx) \phi'(\hat\ww_r^{(k),T} \xx') \hat a_r^{(k),2} \xx^T \xx' \bigr).
\end{eqnarray*}
If we consider training with the continuous-time GD this kernel drives the evolution of the model, see SM~B:
\begin{equation}
    \dot f_d^{(t)}(\xx') =
    -\hat\eta \EE_{\xx,y} \nabla_f^{(t)} \ell(\xx,y) \, \Theta_d^{(t)}(\xx,\xx'),
    \label{eq:f_evolution}
\end{equation}
where we have taken $\hat\eta_a = \hat\eta_w = \hat\eta$.

For a finite $d$ the NTK is a random variable, however when $\sigma \propto d^{-1/2}$, $\Theta_d^{(0)}$ converges to a deterministic non-degenerate limit kernel $\Theta_\infty$ due to the Law of Large Numbers.
Moreover, if $\delta\hat\ww_r^{(k)}$ and $\delta\hat a_r^{(k)}$ vanish with $d$ (i.e. $q_{a/w}^{(k)} < 0$), then $\hat\ww_r^{(k)}$ and $\hat a_r^{(k)}$ converge to $\hat\ww_r^{(0)}$ and $\hat a_r^{(0)}$ respectively.
Hence $\Theta_d^{(k)}$ converges to the same deterministic non-degenerate limit kernel $\Theta_\infty$.

However $\Theta_\infty$ becomes uniformly zero when $q_\sigma < -1/2$.
Given $q_{a/w}^{(k)} < 0$, $\Theta_d^{(k)}$ still converges to $\Theta_\infty \equiv 0$.
Nevertheless, if $\tilde q = \tilde q_a = \tilde q_w = -1 - 2 q_\sigma$, then a new kernel $\tilde \Theta_d^{(k)} = \hat\eta \Theta_d^{(k)}$ converges to a non-vanishing deterministic limit kernel $\tilde\Theta_\infty$.
The dynamics of the limit model is then driven by the above-mentioned limit kernel:
\begin{equation}
    \dot f_\infty^{(t)}(\xx') =
    -\EE_{\xx,y} \nabla_f^{(t)} \ell(\xx,y) \, \tilde\Theta_\infty(\xx,\xx').
    \label{eq:f_inf_evolution}
\end{equation}
Moreover, similar evolution equation holds also for the discrete-time dynamics, see again SM~B:
\begin{equation}
    f_\infty^{(k+1)}(\xx') - f_\infty^{(k)}(\xx') =
    -\EE_{\xx,y} \nabla_f^{(k)} \ell(\xx,y) \, \tilde\Theta_\infty(\xx,\xx').
    \label{eq:f_inf_evolution_discrete}
\end{equation}

Note also that if $q_\sigma < -1/2$ then the limit model vanishes at the initialization due to the Central Limit Theorem: 
\begin{equation*}
    f_d^{(0)} =
    \sigma \sum_{r=1}^d \hat a_r \phi(\hat\ww_r^T \xx) \sim
    d^{q_\sigma + 1/2}
    \; \text{for $d\to\infty$.}
\end{equation*}
We shall refer scalings for which $q_\sigma \in (-1,-1/2)$ as "intermediate".
Since $f_\infty^{(0)}$ is zero for the intermediate scalings while it is not for the NTK scaling, limits induced by intermediate scalings do not coincide with the NTK limit.
As we show in the next section, intermediate limits do not coincide with the MF limit either.
Nevertheless, despite the altered initialization, the limit dynamics for intermediate scalings is still driven by the kernel similar to the NTK case: see eq. (\ref{eq:f_inf_evolution_discrete}).
Note that this "intermediate" limit dynamics is the same for any $q_\sigma \in (-1,-1/2)$.
\citet{chizat2019lazy} have already noted that taking $q_\sigma \in (-1,-1/2]$ leads to the so-called "lazy-training" regime that in our terminology reads simply as $q_{a/w}^{(k)} < 0$.

\subsubsection{Mean-field limit}
\label{sec:mf_limit}

If we take $q_a^{(1)} = q_w^{(1)} = 0$, then again, equations (\ref{eq:qk_1hid}) imply $q_a^{(k)} = q_a^{(1)}$ and $q_w^{(k)} = q_w^{(1)}$ $\forall k \geq 1$.
In this case we conclude that $\varkappa_\emptyset^{(k)} = \varkappa_{aw}^{(k)} = \varkappa_{a/w}^{(k)} = 1$ (see SM~D).

Conditions (\ref{eq:non-trivialness_condition1}) and (\ref{eq:non-trivialness_condition2}) than imply $q_\sigma = -1$.
It follows than from (\ref{eq:q1_1hid}) that $\tilde q_a = \tilde q_w = 1$.
Taking $\sigma = \sigma^* d^{-1}$ and $\hat\eta_{a/w} = \hat\eta^* d$ allows us to write the gradient descent step as a measure evolution equation.

Indeed, consider a weight-space measure:
$\mu_d^{(k)} = \frac{1}{d} \sum_{r=1}^d \delta_{\hat a_r^{(k)}} \otimes \delta_{\hat\ww_k^{(k)}}$.
Given this, a neural network output can be represented as
$f_d^{(k)}(\xx) = \sigma^* \int \hat a \phi(\hat\ww^{T} \xx) \, \mu_d^{(k)}(d\hat a, d\hat\ww)$ while the gradient descent step can be represented as 
$\mu_d^{(k+1)} = \TT(\mu_d^{(k)}; \; \eta^*, \sigma^*)$.

$\mu_d^{(0)}$ converges to $\mu_\infty^{(0)} = \NN_{1+d_0}(0, I)$ in the limit of infinite width.
Since $\eta^*$ and $\sigma^*$ are constants, the evolution of this limit measure is still driven by the same transition operator $\TT$: $\mu_\infty^{(k+1)} = \TT(\mu_\infty^{(k)}; \; \eta^*, \sigma^*)$.
In SM~F we prove that than $\mu_d^{(k)}$ converges to $\mu_\infty^{(k)} = \TT^k(\mu_d^{(0)})$ and $f_d^{(k)}$ converges to a finite $f_\infty^{(k)} = \sigma^* \int \hat a \phi(\hat\ww^T \xx) \, \mu_\infty^{(k)}(d\hat a, d\hat\ww)$:
\begin{thm}[Informal version of Corollary~1 in SM~F]
    If $\sigma \propto d^{-1}$, $\hat\eta_{a/w} \propto d$, and $\ell$, $\phi$, and the data distribution are sufficiently regular, then there exist limits in probability as $d\to\infty$ for $\mu_d^{(k)}$ and for $f_d^{(k)}(\xx)$ $\forall \xx$ $\forall k \geq 0$.
    \label{thm:informal}
\end{thm}

This theorem states the convergence of the discrete-time dynamics of a finite-width model to a discrete-time dynamics of a limit model.
We call the corresponding limit model \emph{a discrete-time mean-field limit.}

This limit differs from those considered in prior works.
Indeed, previous studies on the mean-field theory resulted in a continuous-time dynamics for the limit model.
For example, \cite{sirignano2018lln} assume $\hat\eta \propto 1$.
They prove that in this setup $\mu_d^{t d}$ converges to a continuous-time measure-valued process $\nu^t$ for $t \in \RR$.
The limit process $\nu_t$ is driven by a certain integro-differentiable equation.
In contrast, in our case $\mu_\infty^{(k)}$ is driven by a discrete-time process.
Other works (e.g. \citet{mei2018mf,mei2019mf}) assume $\hat\eta = o(d)$ and also consider a continuous-time evolution for a limit model.

At the same time \citet{rotskoff2019mf} and \citet{chizat2018mf} assume a learning rate scaling similar to ours but they consider a continuous-time gradient descent dynamics for the finite-width net.

Note that if $q_{a/w}^{(k)} < 0$ (as for NTK and intermediate scalings), then $\delta\hat a^{(k)}$ and $\delta\hat\ww^{(k)}$ vanish as $d \to \infty$, hence $\mu_d^{(k)}$ converges to $\mu_\infty^{(0)} = \NN_{1+d_0}(0,I)$.
This means that in this case we cannot represent the dynamics of the limit model $f_\infty^{(k)}$ in terms of the dynamics of the limit measure, hence this case is out of the scope of the MF theory.

On the other hand, if $q_a^{(k)} = q_w^{(k)} = 0$, then a deterministic limit $\lim_{d\to\infty} \tilde\Theta_d^{(k)}(\xx,\xx')$ still exists due to the Law of Large Numbers, however this limit depends on step $k$ since $\phi'(\hat\ww^{(k),T} \xx) \nrightarrow \phi'(\hat\ww^{(0),T} \xx)$.
Hence the dynamics of a limit model $f_\infty^{(k)}$ in the mean-field limit cannot be described in terms of a constant deterministic kernel.


So far we have considered two cases: $q_{a/w}^{(1)} < 0$ and $q_{a/w}^{(1)} = 0$.
We elaborate other possible cases in SM~E.

\subsection{Infinite-width limits as approximations for finite-width nets}

In the previous section we have introduced a family of scalings leading to different limit models.
Limit models can be easier to study mathematically: for example, in the NTK limit the training process converges to a kernel method.
If we show that a limit model approximates the original one well, we can substitute the latter with the former in our theoretical considerations.


Notice that conditions (\ref{eq:non-trivialness_condition1}) and (\ref{eq:non-trivialness_condition2}) allow some of (but not all of) $q_{f,\emptyset}^{(k)}$, $q_{f,a}^{(k)}$, $q_{f,w}^{(k)}$ and $q_{f,aw}^{(k)}$ to be less than zero.
This means that corresponding terms of decomposition (\ref{eq:f_decomposition_1hid}) vanish as $d \to \infty$.
However for $d = d^*$, where $d^* < \infty$ is the width of a "reference" model, all of these terms are present.
If we assume that indeed all of these terms obey power-laws with respect to $d$ (which is a reasonable assumption for large $d$), then we can conclude that the fewer terms vanish as $d \to \infty$, the better the corresponding limit approximates the original finite-width net.
We validate this assumption for the above-mentioned scalings in SM~C.

One can see that for the NTK limit we have $q_{f,\emptyset}^{(k)} = q_{f,a}^{(k)} = q_{f,a}^{(k)} = 0$, hence the first three terms of decomposition (\ref{eq:f_decomposition_1hid}) are preserved as $d \to \infty$, however $q_{f,aw}^{(k)} = -1$.
In Figure~\ref{fig:mf_ntk_1hid_cifar2_sgd_test_loss_and_var_f} (center) we empirically check that this is indeed the case.
One can notice however that the last term, which is not preserved, vanishes as $\hat\eta \to 0$ faster than $f_{d,a}^{(k)}$ and $f_{d,w}^{(k)}$.
This reflects the fact that originally the NTK limit was derived for the continuous-time gradient descent for which the learning rate is effectively zero.

Note also that if $q_{a/w}^{(1)} < 0$ (for which the NTK scaling is a special case), then $q_{f,aw}^{(k)} < 0$ (see above), hence the last term of decomposition (\ref{eq:f_decomposition_1hid}) always vanishes in this case.
Hence the NTK scaling should provide the most reasonable approximation for finite-width nets among all scalings in this class.
For comparison, we also consider the intermediate scaling $q_\sigma = -3/4$, $\tilde q_{a/w} = 1/2$ for which $q_{f,\emptyset}^{(k)} = -1/4$, $q_{f,a}^{(k)} = q_{f,w}^{(k)} = 0$, and $q_{f,aw}^{(k)} = -1/2$ for $k \geq 1$.

In contrast, for the MF limit we have $q_{f,\emptyset}^{(k)} = q_{f,a}^{(k)} = q_{f,w}^{(k)} = q_{f,aw}^{(k)} = 0$ for $k \geq 1$.
Hence we expect all the terms of decomposition (\ref{eq:f_decomposition_1hid}) to be preserved as $d \to \infty$.
We check this claim empirically in Figure~\ref{fig:mf_ntk_1hid_cifar2_sgd_test_loss_and_var_f}, center.

We also found it interesting to plot the case of the "default" scaling: $\sigma \propto d^{-1/2}$ and $\eta_{a/w} \propto 1$ (black curves).
It corresponds to the situation when we make our network wider while keeping learning rates in the original parameterization constant.
In this case $\hat\eta_a \propto d$, $\hat\eta_w \propto 1$, hence $\tilde q_a = 1$ and $\tilde q_w = 0$.

We compare final test losses for the above-mentioned scalings in Figure~\ref{fig:mf_ntk_1hid_cifar2_sgd_test_loss_and_var_f}, left.
As we see, all scalings except the default one result in finite limits for the loss while the default one diverges.
As we see in Figure~\ref{fig:mf_ntk_1hid_cifar2_sgd_test_loss_and_var_f} (right), the mean-field limit tracks the learning dynamics of the reference network better than the other limits considered.
It is interesting to note also that as the learning dynamics shows, MF and intermediate limits are deterministic while the NTK limit, as well as the reference model, are not.
This is because the model at the initialization converges to zero for the first two cases.
Also, this is the reason why the NTK limit becomes a better approximation for a finite-width net if learning rates are small enough (see Figure~3 in SM~H).
In this case the term $f_{d,aw}^{(k)}$, which is not preserved in the NTK limit, becomes negligible already for the reference network.

\begin{figure*}[t]
    \centering
    \includegraphics[width=0.32\textwidth]{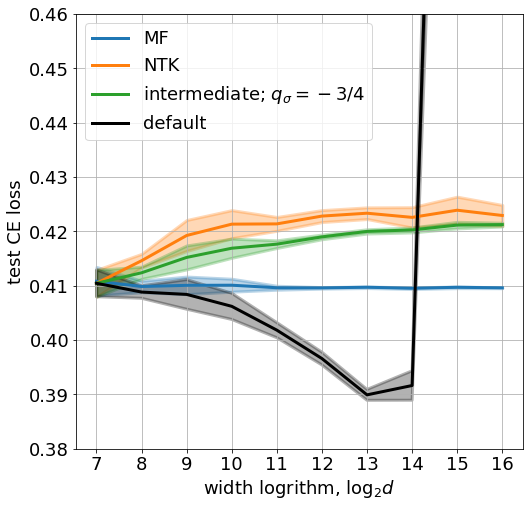}
    \includegraphics[width=0.32\textwidth]{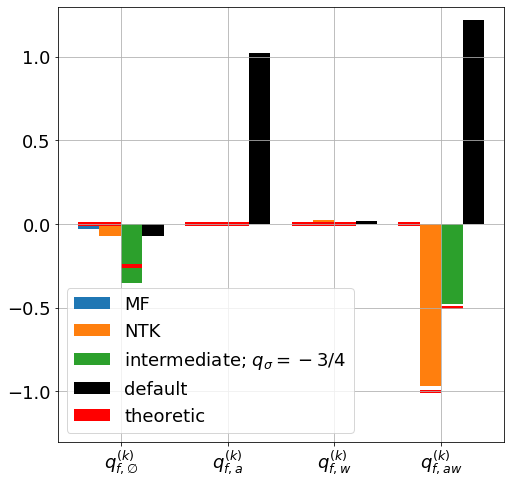}
    \includegraphics[width=0.32\textwidth]{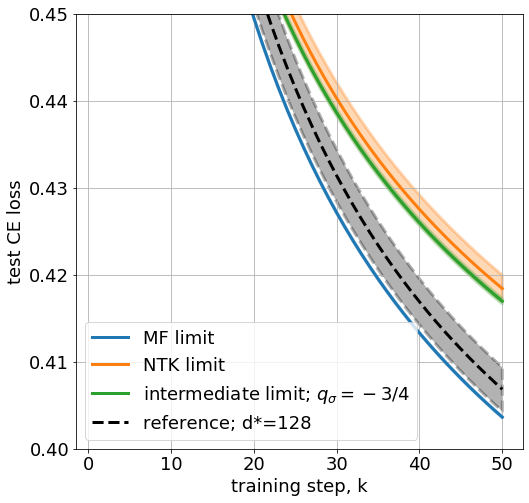}
    \caption{\textbf{MF, NTK and intermediate scalings result in non-trivial limit models for a single layer neural net. A limit model induced by the intermediate scaling differs from both MF and NTK limits.} \textit{Left:} a final test cross entropy (CE) loss as a function of the width $d$. MF, NTK and intermediate scalings converge but the default scaling does not. The MF limit approximates the reference finite-width network better than all other limits. \textit{Center:} numerical estimates for the exponents of the decomposition (\ref{eq:f_decomposition_1hid}) terms as well as their theoretical values (denoted by red ticks). We see that for the default scaling some of the exponents are positive, hence corresponding decomposition terms diverge. For the MF limit all of the exponents are zeros, meaning all of the decomposition terms are preserved. Also, we see that our numerical experiments match the theory well. \textit{Right:} the test CE loss as a function of training step $k$ for the reference net and its limits. We see that 1) the MF limit best matches the reference, 2) the NTK limit is not deterministic while the intermediate limit is. This is because the model at the initialization converges to zero for the intermediate scaling. \textit{Setup:} We train a 1-hidden layer net on a subset of CIFAR2 (a dataset of the first two classes of CIFAR10) of size 1000 with gradient descent. We take a reference net of width $d^* = 2^7 = 128$ trained with unscaled reference learning rates $\eta_a^* = \eta_w^* = 0.02$ and scale its hyperparameters according to MF (blue curves), NTK (orange curves), and intermediate scaling with $q_\sigma=-3/4$ (green curves, see text). We also make a plot for the case when we do not scale our learning rates (black curves) and scale standard deviations at the initialization as the initialization scheme of \citet{he2015init} suggests. See SM~A for further details.}
    \label{fig:mf_ntk_1hid_cifar2_sgd_test_loss_and_var_f}
\end{figure*}

\section{Training a multi-layer net}
\label{sec:Hhid_gd}

While reasoning about non-trivial limits for multi-layered nets is more technically involved, some qualitative results are still possible.
For instance, we show that a (discrete-time) mean-field limit is vanishing for networks with more than three hidden layers.
However, such a limit seems to exist if the network is trained with RMSProp.

Consider a multi-layered network with all hidden layers having width $d$:
\begin{equation}
    f(\xx; \aa, V^{1:H}, W) =
    \sum_{r_H=1}^d a_{r_H} \phi(f^H_{r_H}(\xx; V^{1:H}, W)),
    \label{eq:f_Hhid}
\end{equation}
where
\begin{equation*}
    f^{h+1}_{r_{h+1}}(\xx; V^{1:h+1}, W) =
    \sum_{r_h=1}^d v^{h+1}_{r_{h+1} r_h} \phi(f^h_{r_h}(\xx; V^{1:h}, W)),
\end{equation*}
\begin{equation*}
    f^0_{r_0}(\xx, W) =
    \ww_{r_0}^T \xx.
\end{equation*}
Here again, all quantities are initialized with zero-mean gaussians: $a_{r_H}^{(0)} \sim \NN(0, \sigma_a^2)$, $\ww_{r_0}^{(0)} \sim \NN(0, \sigma_w^2 I)$, and $v^{h,(0)}_{r_h r_{h-1}} \sim \NN(0, \sigma_{v^h}^2)$.

We perform a gradient descent step for the parameters $\aa$, $V^{1:H}$, $W$ with learning rates $\eta_a$, $\eta_{v^{1:H}}$, and $\eta_w$ respectively.
We introduce scaled quantities in the similar manner as for the single hidden layer case:
\begin{equation*}
    \hat a_{r_H}^{(k)} = \frac{a_{r_H}^{(k)}}{\sigma_a}, \quad
    \hat v_{r_h r_{h-1}}^{h,(k)} = \frac{v_{r_h r_{h-1}}^{h,(k)}}{\sigma_{v^h}}, \quad
    \hat\ww_{r_0}^{(k)} = \frac{\ww_{r_0}^{(k)}}{\sigma_w},    
\end{equation*}
\begin{equation*}
    \hat \eta_a = \frac{\eta_a}{\sigma_a^2}, \quad
    \hat \eta_{v^h} = \frac{\eta_{v^h}}{\sigma_{v^h}^2}, \quad
    \hat \eta_w = \frac{\eta_w}{\sigma_w^2}.    
\end{equation*}

Given this, the gradient descent step on the scaled quantities writes as follows:
\begin{equation*}
    \Delta\hat a_{r_H}^{(k)} =
    -\hat\eta_a \sigma^{H+1} \EE \nabla_f^{(k)} \ell(\xx,y) \; \phi(\hat f^{H,(k)}_{r_H}(\xx)),
\end{equation*}
\begin{equation*}
    \Delta\hat v^{H,(k)}_{r_H r_{H-1}} =
    -\hat\eta_{v^H} \sigma^{H+1} \EE \nabla_f^{(k)} \ell(\xx,y) \; \hat a_{r_H}^{(k)} \phi(\hat f^{H-1,(k)}_{r_{H-1}}(\xx)),
\end{equation*}
$$\ldots$$
\begin{multline*}
    \Delta\hat \ww_{r_0}^{(k)} =
    -\hat \eta_w \sigma^{H+1} \EE_{\xx,y} \nabla_f^{(k)} \ell(\xx,y)
    \times\\\times
    \sum_{r_H=1}^d \hat a^{(k)}_{r_H} \phi'(\hat f^{H,(k)}_{r_H}(\xx)) 
    \times\\\times
    \sum_{r_{H-1}=1}^d \hat v^{H,(k)}_{r_H r_{H-1}} \phi'(\hat f^{H-1,(k)}_{r_{H-1}}(\xx)) 
    \times\ldots\\\ldots\times
    \sum_{r_{1}=1}^d \hat v^{2,(k)}_{r_2 r_{1}} \phi'(\hat f^{1,(k)}_{r_1}(\xx)) 
    \hat v^{1,(k)}_{r_1 r_0} \phi'(\hat \ww_{r_0}^{(k),T} \xx) \xx.
\end{multline*}
\begin{equation}
    \hat a^{(0)}_{r_H} \sim \NN(0, 1), \;
    \hat v^{h,(0)}_{r_h r_{h-1}} \sim \NN(0, 1), \;
    \hat\ww^{(0)}_{r_0} \sim \NN(0, I),
    \label{eq:gd_dynamics_Hhid}
\end{equation}
where we have denoted $\hat f^{h,(k)}_{r_h}(\xx) = f^h_{r_h}(\xx; \hat V^{(k),1:h}, \hat W^{(k)})$ and $\sigma = (\sigma_a \sigma_{v^H} \ldots \sigma_{v^1} \sigma_w)^{1/(H+1)}$.

\subsection{MF scaling leads to a trivial discrete-time MF limit}

As we have noted in Section~\ref{sec:mf_limit}, the mean-field theory describes a state of a neural network with a measure in the weight space $\mu$; similarly, it describes a networks' learning dynamics as an evolution of this measure.
In particular, this means that weight updates cannot depend explicitly on the width $d$.
Indeed, if they grow with $d$, then for some measure $\mu_\infty$ with infinite number of atoms this measure will diverge after a single gradient step.
Similarly, if they vanish with $d$, then for a measure with an infinite number of atoms this measure will not evolve with gradient steps.
Since we consider a polynomial dependence on $d$ for our hyperparameters, our dynamics should not depend on $d$ explicitly.

It is not obvious how to properly define a weight-space measure in the case of multiple hidden layers; the discussion in Section 3.3 of \citet{sirignano2019deep_mf}; see also \citet{fang2019deep_mf}.
However, if we manage to define it properly, then each sum in the dynamics equation (\ref{eq:gd_dynamics_Hhid}) will be substituted with an integral over the measure.
Each such integral will contribute a $d$ factor to the corresponding equation.
Hence in order to have a learning dynamics independent on $d$ we should have:
\begin{equation*}
    \hat\eta_{a/w} \sigma^{H+1} d^{H} \propto 1, \quad
    \hat\eta_{v^h} \sigma^{H+1} d^{H-1} \propto 1,
\end{equation*}
because there are $H$ sums in the dynamics equation for $\hat a$ and $\hat\ww$, and $H-1$ sums for $\hat v^{1:H}$.
Similarly, since the network output should not depend on $d$, we should also have:
\begin{equation*}
    \sigma^{H+1} d^{H+1} \propto 1.
\end{equation*}
From this follows that $\sigma \propto d^{-1}$, $\hat\eta_{a/w} \propto d$, and $\hat\eta_{v^h} \propto d^2$.

As we show in SM~G, for $H \geq 2$ this scaling leads to a vanishing limit: $f_d^{(k)}(\xx) \to 0$ as $d\to\infty$ $\forall \xx$ $\forall k \geq 0$.
The intuition behind this result is simple: if $H \geq 2$ and $\phi(z) \sim z$ for $z \to 0$, then given the scaling above all of the weight increments vanish as $d\to\infty$ for $k=0$.
This means that the learning process cannot start in the limit of large $d$.
We validate this claim empirically for $H = 2$ in Figure~\ref{fig:mf_ntk_3hid_cifar2_sgd_test_loss_and_q_f}, center.
In contrast, for the NTK scaling, which corresponds to $\sigma \propto d^{-1/2}$ and $\hat\eta_{a/v^h/w} \propto 1$, not all of the terms vanish.
Nevertheless, if $H = 1$, a non-trivial mean-field limit seems to exist as our experiments demonstrate: see Figure~\ref{fig:mf_ntk_3hid_cifar2_sgd_test_loss_and_q_f}, left.

Note that this result does not drive away the possibility of constructing a meaningful continuous-time MF limit, for which the limit dynamics is driven by an ODE. 
Also, we expect a meaningful MF limit to be possible for non-linearities that do not vanish near zero (e.g. for sigmoid).

\subsection{Training a multi-layer net with RMSProp}
\label{sec:Hhid_rmsprop}

Up to this point we have considered a GD training.
Consider now training with RMSProp which updates the weights with normalized gradients.
We show that in this case a mean-field limit exists and it is not trivial for any $H \geq 0$.

For the RMSProp training gradient updates look as follows:
\begin{equation}
    \Delta\theta^{(k)} = 
    \theta^{(k+1)} - \theta^{(k)} = 
    -\eta_\theta \frac{\nabla_\theta^{(k)}}{\mathrm{RMS}_\theta^{(k)}},
    \label{eq:weight_update_rmsprop}
\end{equation}
where $\theta \in \{a_{r_H}, v^H_{r_H r_{H-1}}, \ldots, v^1_{r_1 r_0}, \ww_{r_0}\}$. 
Here we have used following shorthands:
\begin{equation*}
    \nabla_\theta^{(k)} = \nabla_\theta \LL(\Theta^{(k)}),
\end{equation*}
where $\Theta^{(k)} = \{a_{r_H}^{(k)}, v^{H,(k)}_{r_H r_{H-1}}, \ldots, v^{1,(k)}_{r_1 r_0}, \ww_{r_0}^{(k)}\}$, and
\begin{equation*}
    \mathrm{RMS}_\theta^{(k)} = \sqrt{\sum_{k'=0}^k \beta^{k-k'} \nabla_\theta^{(k')} \odot \nabla_\theta^{(k')}} \quad \text{for $\beta \in (0,1)$.}
\end{equation*}

Similarly to the GD case, we divide equation (\ref{eq:weight_update_rmsprop}) by the standard deviation $\sigma_\theta$ of the initialization of the weight $\theta$:
\begin{equation*}
    \Delta\hat\theta^{(k)} = -\frac{\eta_\theta}{\sigma_\theta} \frac{\nabla_{\hat\theta}^{(k)}}{\mathrm{RMS}_{\hat\theta}^{(k)}},
\end{equation*}
where $\nabla_{\hat\theta}^{(k)}$ and $\mathrm{RMS}_{\hat\theta}^{(k)}$ are defined similarly as above.

In this case we define scaled learning rates differently compared to the GD case: $\hat\eta_\theta = {\eta_\theta}/{\sigma_\theta}$.



As noted above, the mean-field analysis requires weight updates not to depend on $d$ explicitly.
Since our weight update rule uses normalized gradients, this condition reads simply as $\hat\eta_\theta \propto 1$ for all weights $\theta$ and $\sigma \propto d^{-1}$ since the model output $f[\mu_d; \xx]$ should not depend on $d$ explicitly.

Using similar reasoning as before (namely, weight increments should decay as $d^{-1/2}$) we can also define the NTK scaling: $\hat\eta_\theta \propto 1$ for all $\theta$ and $\sigma \propto d^{-1/2}$.
We compare these two limits in Figure~\ref{fig:mf_ntk_3hid_cifar2_sgd_test_loss_and_q_f}, right.
Notice that similar to the single hidden layer case, the NTK limit preserves terms with low-order dependencies on learning rates (i.e. $f_{d,\emptyset}^{(k)}$, $f_{d,a/v^h/w}^{(k)}$), while the MF limit, being now non-vanishing, preserves terms with higher-order dependencies on them.

\begin{figure*}[t]
    \centering
    \includegraphics[width=0.32\textwidth]{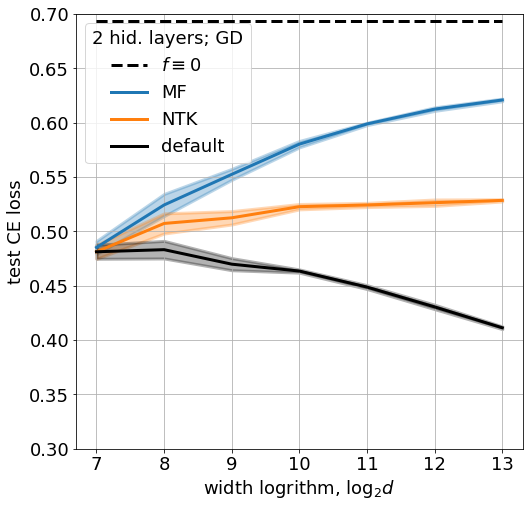}
    \includegraphics[width=0.32\textwidth]{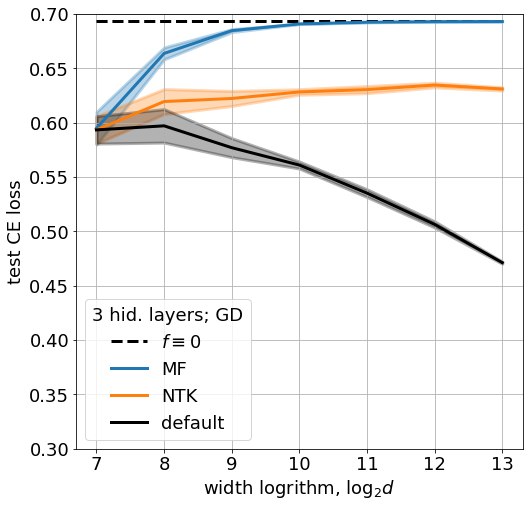}
    \includegraphics[width=0.32\textwidth]{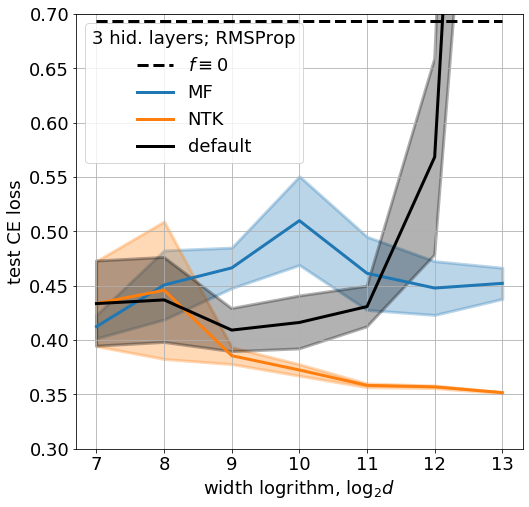}
    \\
    \includegraphics[width=0.32\textwidth]{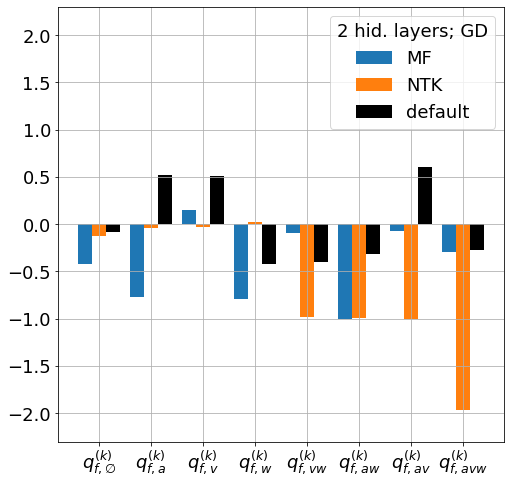}
    \includegraphics[width=0.32\textwidth]{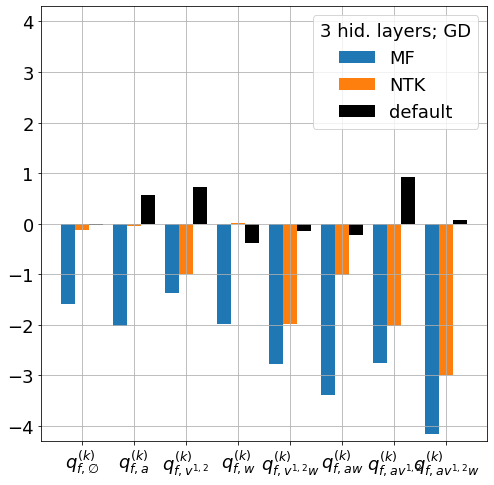}
    \includegraphics[width=0.32\textwidth]{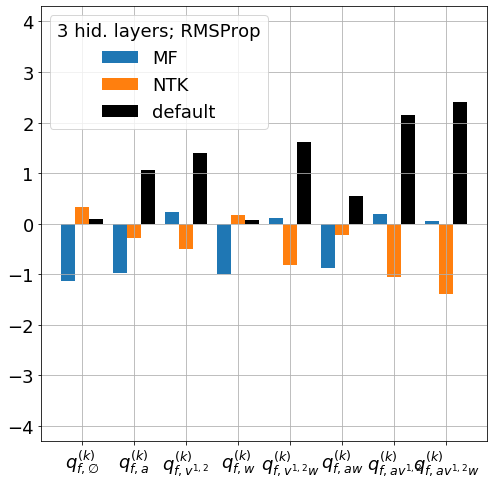}
    \caption{\textbf{MF and NTK limits for multilayer networks.}
    \textit{Top row:} the final test cross-entropy (CE) loss as a function of width $d$. \textit{Bottom row:} numerical estimates for exponents of terms of the decomposition of $f^{(k)}$, similar to eq.(\ref{eq:f_decomposition_1hid}). All of these terms vanish for a network with (at least) three hidden layers in the MF limit, however, this is not the case when the number of hidden layers is two. Nevertheless, if we consider training with RMSProp, the MF limit becomes non-vanishing. For the NTK scaling, not all of the decomposition terms vanish in any case, however, some of them do, indicating possible discrepancies between the reference net and its NTK limit.
    \textit{Setup:} We train a multi-layer net on a subset of CIFAR2 (a dataset of first two classes of CIFAR10) of size 1000 with either a plain gradient descent or RMSProp. We take a reference net of width $d^* = 2^7 = 128$ trained with (unscaled) reference learning rates $\eta_a^* = \eta_w^* = 0.02$ for GD and $\eta_a^* = \eta_w^* = 0.0002$ for RMSProp, and scale its hyperparameters according to MF (blue curves) and NTK (orange curves) scalings. We also make a plot for the case when we do not scale our learning rates (black curves) while scaling standard deviations at the initialization as the initialization scheme of \citet{he2015init} suggests. See SM~A for further details.}
    \label{fig:mf_ntk_3hid_cifar2_sgd_test_loss_and_q_f}
\end{figure*}

\section{Conclusions}
\label{sec:conclusions}

There are two different theories that study neural nets in the limit of infinite width: a mean-field theory and a kernel theory.
These theories imply that if certain conditions are fulfilled, corresponding infinite-width limits are non-trivial, i.e. the resulting function neither explodes nor vanishes and the learning process does not get stuck as the width goes to infinity.

In our study we derive a set of sufficient conditions on the scaling of hyperparameters (weight initialization variances and learning rates) with width to ensure that we reach a non-trivial limit when the width goes to infinity.
Solutions under these conditions include scalings that correspond to mean-field and NTK limits, as well as a family of scalings that lead to a limit model different from these two.

We propose a decomposition of our model and show that some of its terms may vanish for large width.
We argue that a limit provides a more reasonable approximation for a finite-width net if as few of these terms vanish as possible.

Our analysis out of the box suggests a discrete-time MF limit which, to the best of our knowledge, has not been covered by the existing literature yet.
We prove a convergence theorem for it and show that it provides a more reasonable approximation for finite-width nets than the NTK limit as long as learning rates are not too small.


As we show afterwards, a discrete-time mean-field limit appears to be trivial for a network with more than two hidden layers.
Nevertheless, if we train our network with RMSProp instead of GD, the above-mentioned limit becomes non-trivial for any number of hidden layers.


\paragraph*{Acknowledgments}
This work  was  supported  by  National  Technology  Initiative  and  PAO  Sberbank  project  ID0000000007417F630002. We thank Mikhail Burtsev and Ivan Skorokhodov for valuable discussions and suggestions, as well as for help in improving the final version of the text.

\bibliography{main}
\bibliographystyle{icml2020}

\appendix

\section{Experimental details}
\label{sec:appendix_experiments}

We perform our experiments on a feed-forward net with $H+1$ hidden layers with no biases.
We learn our network as a binary classifier on a subset of CIFAR2 dataset (which is a dataset of first two classes of CIFAR10) of size 1000.
We train our network for 50 epochs to minimize the binary cross-entropy loss and report the final cross-entropy loss on a full test set (of size 2000).
We repeat our experiments for 5 random seeds and report means and standard deviations on our plots.
We experiment with other setups (i.e. using a mini-batch gradient estimation instead of the exact one, using a larger train dataset, using more training steps, learning a multi-class classification problem) in SM~\ref{sec:appendix_other_setups}.
All experiments were conducted on a single NVIDIA GeForce GTX 1080 Ti GPU using pytorch framework \cite{paszke2017automatic}.
Our code is available online: \url{https://github.com/deepmipt/research/tree/master/Infinite_Width_Limits_of_Neural_Classifiers}.

Although our analysis assumes initializing variables with samples from a gaussian, nothing changes if we sample $\sigma \xi$ for $\xi$ being any symmetric random variable with a distribution independent on hyperparameters.

In our experiments we took a network of width $d^* = 2^7 = 128$ and apply a Kaiming uniform initialization scheme \cite{he2015init} to its layers; we call this network a reference network.
Consider a network with a single hidden layer first.
According to the Kaiming initialization scheme, initial weights should have zero mean and standard deviations $\sigma_a^* \propto (d^*)^{-1/2}$ and $\sigma_w^* \propto d_0^{-1/2}$, where $d_0$ is the input dimension which we do not modify.
For this network we take (unscaled!) learning rates $\eta_a^* = \eta_w^* = 0.02$ for the gradient descent training and $\eta_a^* = \eta_w^* = 0.0002$ and $\beta = 0.99$ for the RMSProp training.
After that, we scale the initial weights and the learning rates with width $d$ according to a specific scaling:
\begin{equation*}
    \sigma = 
    \sigma^* \left(\frac{d}{d^*}\right)^{q_\sigma},
    \quad
    \hat\eta_{a/w} = 
    \hat\eta_{a/w}^* \left(\frac{d}{d^*}\right)^{\tilde q_{a/w}}.
\end{equation*}
Since $\sigma = \sigma_a \sigma_w$ and since we apply the (leaky) ReLU non-linearity, we can take
\begin{equation*}
    \sigma_a = 
    \sigma_a^* \left(\frac{d}{d^*}\right)^{q_\sigma},
    \quad
    \sigma_w = 
    \sigma_w^*.
\end{equation*}
Since for GD we have $\hat\eta_{a/w} = \eta_{a/w} / \sigma_{a/w}^2$, then
\begin{equation*}
    \eta_a = 
    \eta_a^* \left(\frac{\sigma_a}{\sigma_a^*}\right)^2 \left(\frac{d}{d^*}\right)^{\tilde q_a} =
    \eta_a^* \left(\frac{d}{d^*}\right)^{\tilde q_a+2q_\sigma},
\end{equation*}
\begin{equation*}
    \eta_w = 
    \eta_w^* \left(\frac{\sigma_w}{\sigma_w^*}\right)^2 \left(\frac{d}{d^*}\right)^{\tilde q_w} =
    \eta_w^* \left(\frac{d}{d^*}\right)^{\tilde q_w}.
\end{equation*}
Similar holds for the multi-layer case.
In this case since $\sigma = (\sigma_a \sigma_{v^H} \ldots \sigma_{v^1} \sigma_w)^{1/(1+H)}$, we can take
\begin{equation*}
    \sigma_{a/v^1/\ldots/v^H} = 
    \sigma_{a/v^1/\ldots/v^H}^* \left(\frac{d}{d^*}\right)^{q_\sigma},
    \quad
    \sigma_w = 
    \sigma_w^*.
\end{equation*}

\section{Dynamics of the limit model for the NTK scaling}
\label{sec:appendix_limiting_dynamics}

First consider a continuous-time gradient descent for a one-hidden layer network in a general form:
\begin{equation*}
    \dot\theta_d^{(t)} = -\hat\eta \EE_{\xx,y} \left.\frac{\partial\ell(y,z)}{\partial z}\right|_{z=f(\xx; \, \theta_d^{(t)})} \frac{\partial f(\xx; \, \theta_d^{(t)})}{\partial\theta_d},
\end{equation*}
where $\theta_d^{(t)} = \{(\hat a_r^{(t)},\hat\ww_r^{(t)})\}_{r=1}^d$ is a sequence of $d$ weights $(\hat a, \hat\ww)$ associated with each neuron at a time-step $t$.
\begin{eqnarray*}
    \dot f(\xx'; \, \theta_d^{(t)}) 
    =
    \left(\frac{\partial f(\xx'; \, \theta_d^{(t)})}{\partial\theta_d}\right)^T \dot\theta_d^{(t)} 
    =\\=
    -\hat\eta \EE \left.\frac{\partial\ell(y,z)}{\partial z}\right|_{z=\langle\ldots\rangle} \left(\frac{\partial f(\xx'; \, \theta_d^{(t)})}{\partial\theta_d}\right)^T \frac{\partial f(\xx; \, \theta_d^{(t)})}{\partial\theta_d}
    =\\=
    -\hat\eta \EE_{\xx,y} \left.\frac{\partial\ell(y,z)}{\partial z}\right|_{z=f(\xx; \, \theta_d^{(t)})} \Theta_d(\xx',\xx; \, \theta_d^{(t)}).
\end{eqnarray*}

Assume the model is scaled as $d^{-1/2}$:
\begin{equation*}
    f(\xx; \, \theta_d^{(t)}) =
    d^{-1/2} \sum_{r=1}^d \hat a_r^{(t)} \phi(\hat\ww_r^{(t),T} \xx).
\end{equation*}
Then a neural tangent kernel is written as follows:
\begin{eqnarray*}
    \Theta_d(\xx',\xx; \, \theta_d^{(t)}) =
    d^{-1} \sum_{r=1}^d \Bigl( \phi(\hat\ww_r^{(t),T} \xx) \phi(\hat\ww_r^{(t),T} \xx') 
    +\\+ 
    \hat a_r^2 \phi'(\hat\ww_r^{(t),T} \xx) \phi'(\hat\ww_r^{(t),T} \xx') \xx^T \xx' \Bigr).
\end{eqnarray*}
If moreover $\hat\eta = \const$, then for a fixed $t$ independent of $d$ $\hat a^{(t)} \to \hat a^{(0)}$ and $\hat\ww^{(t)} \to \hat\ww^{(0)}$.
Hence due to the Law of Large Numbers $\Theta_d(\xx',\xx; \, \theta_d^{(t)}) \to \Theta_\infty(\xx',\xx)$, where
\begin{eqnarray*}
    \Theta_\infty(\xx',\xx) =
    \EE_{(\hat a, \hat\ww) \sim \NN(0, I_{1+d_0})} \Bigl(\phi(\hat\ww^T \xx) \phi(\hat\ww^T \xx') 
    +\\+ 
    \hat a^2 \phi'(\hat\ww^T \xx) \phi'(\hat\ww^T \xx') \xx^T \xx' \Bigr).
\end{eqnarray*}

In the case of the discrete-time dynamics we have similarly:
\begin{equation*}
    \theta_d^{(k+1)} = 
    \theta_d^{(k)} - \hat\eta \EE_{\xx,y} \left.\frac{\partial\ell(y,z)}{\partial z}\right|_{z=f(\xx; \, \theta_d^{(k)})} \frac{\partial f(\xx; \, \theta_d^{(k)})}{\partial\theta_d}.
\end{equation*}
A classical result of calculus states that there exists a $\xi_d^{(k)} \in [0,1]^{(d_0+1) d}$ such that following holds:
\begin{eqnarray*}
    f(\xx'; \, \theta_d^{(k+1)}) - f(\xx'; \, \theta_d^{(k)})
    =\\=
    \left(\frac{\partial f(\xx'; \, \tilde\theta_d^{(k)})}{\partial\theta_d}\right)^T (\theta_d^{(k+1)} - \theta_d^{(k)})
    =\\=
    -\hat\eta \EE \left.\frac{\partial\ell(y,z)}{\partial z}\right|_{z=\langle\ldots\rangle} \left(\frac{\partial f(\xx'; \, \tilde\theta_d^{(k)})}{\partial\theta_d}\right)^T \frac{\partial f(\xx; \, \theta_d^{(k)})}{\partial\theta_d}
    =\\=
    -\hat\eta \EE_{\xx,y} \left.\frac{\partial\ell(y,z)}{\partial z}\right|_{z=f(\xx; \, \theta_d^{(t)})} \Theta_d(\xx',\xx; \, \theta_d^{(k)}, \tilde\theta_d^{(k)}).
\end{eqnarray*}
where $\tilde\theta_d^{(k)} = \theta_d^{(k+1)} \odot \xi_d^{(k)} + \theta_d^{(k)} \odot (1-\xi_d^{(k)})$, and we have abused notation by redefining $\Theta_d$.
Nevertheless, in this case $\Theta_d(\xx',\xx; \, \theta_d^{(k)}, \tilde\theta_d^{(k)})$ still converges to $\Theta_\infty(\xx',\xx)$ defined above for the same reasons as above.

\section{Validation of the power-law asumptions}
\label{sec:appendix_sanity_checks}

In Section~3 we have introduced power-law assumptions for weight increments and for terms of the model decomposition:
\begin{equation}
    |\delta\hat a_r^{(k)}| \propto d^{q_a^{(k)}}, \quad
    \|\delta\hat\ww_r^{(k)}\| \propto d^{q_w^{(k)}};
    \label{eq:increments_power_law_appendix}
\end{equation}
\begin{equation}
    f_{d,\emptyset}^{(k)}(\xx) \propto d^{q_{f,\emptyset}^{(k)}}, \;
    f_{d,a/w}^{(k)}(\xx) \propto d^{q_{f,a/w}^{(k)}}, \;
    f_{d,aw}^{(k)}(\xx) \propto d^{q_{f,aw}^{(k)}}.
    \label{eq:f_terms_power_law_appendix}
\end{equation}
After that, we have derived corresponding exponents for two cases: $q_{a/w}^{(1)} = q_\sigma + \tilde q_{a/w} < 0$ and $q_{a/w}^{(1)} = q_\sigma + \tilde q_{a/w} = 0$, where $q_\sigma$ is an exponent for $\sigma$ and $\tilde q_{a/w}$ are exponents for learning rates:
\begin{equation*}
    \sigma \propto d^{q_\sigma}, \quad
    \hat\eta_{a/w} \propto d^{\tilde q_{a/w}}.
\end{equation*}

In order to have a non-vanishing non-diverging limit model $f_\infty^{(k)}$ that does not coincide with its initialization $f_\infty^{(0)}$, we have derived a set of conditions: see Section~3.
For the first case these conditions were the following:
\begin{equation*}
    q_\sigma \in (-1,-1/2],
\end{equation*}
\begin{equation*}
    q_{a/w}^{(1)} \leq -1 - q_\sigma, \quad
    \max(q_a^{(1)}, q_w^{(1)}) = -1 - q_\sigma,
\end{equation*}
while for the second case they were:
\begin{equation*}
    q_\sigma = -1, \quad 
    q_{a/w}^{(1)} = 0.
\end{equation*}
The MF scaling is exactly the second case, while the NTK scaling corresponds to the first case: $q_\sigma = q_a^{(1)} = q_w^{(1)} = -1/2$.
We have refered a family of scalings $q_\sigma \in (-1,-1/2)$ and $q_a^{(1)} = q_w^{(1)} = -1 - q_\sigma$ as "intermediate".

As we have also derived in Section~3, for both cases $q_{a/w}^{(k)} = q_{a/w}^{(1)}$ $\forall k \geq 1$.

Here we validate power-law assumptions (\ref{eq:increments_power_law_appendix}) as well as derived values for corresponding exponents for the three special cases noted above: MF, NTK and intermediate scalings, see Figure \ref{fig:1hid_cifar2_sgd_da_dw_q}.
We train a one hidden layer network with the gradient descent for 50 epochs; see SM~\ref{sec:appendix_experiments} for further details.
We take norms of final learned weight increments and average them over hidden neurons:
\begin{equation*}
    \mathrm{av.} \, |\delta\hat a^{(k)}| = \frac{1}{d} \sum_{r=1}^d |\delta\hat a_r^{(k)}|,
\end{equation*}
\begin{equation*}
    \mathrm{av.} \, \|\delta\hat\ww^{(k)}\|_2 = \frac{1}{d} \sum_{r=1}^d \|\delta\hat\ww_r^{(k)}\|_2.
\end{equation*}
We then plot these values as functions of width $d$.

As one can see on left and center plots, weight increments as functions of width are very well fitted with power-laws for both input and output layers.
Right plot matches numerical estimates for corresponding exponents $q_a^{(k)}$ and $q_w^{(k)}$ with their theoretical values (denoted by red ticks).
Here we notice a reasonable coincidence between them.

In order to validate a power-law assumption for model decomposition terms (\ref{eq:f_terms_power_law_appendix}), we compute the variance with respect to the data distribution for each decomposition term.
The reason to consider variances instead of decomposition terms themselves is that these terms are functions of $\xx$.
If we just fix a random $\xx$, then the numerical estimate for, say, $f_{d,a}^{(k)}(\xx)$ can be noisy.
Hence it is better to plot some statistics of these terms with respect to the data, hoping that this statistics will be more robust, which is true e.g. for expectation.
However, since we consider a binary classification problem with balanced classes, we are likely to have $\EE_\xx f_d^{(k)}(\xx) \approx 0$.
Because of this, we are afraid to have all of the decomposition terms to be approximately zeros in expectation.
For this reason, we consider a variance instead of the expectation.
Note that $f_d^{(k)} \propto d^{q_f^{(k)}}$ implies $\Var_x f_d^{(k)} \propto d^{2 q_f^{(k)}}$.

As we see in Figure \ref{fig:1hid_cifar2_sgd_var_f_decomp}, variances of all of the model decomposition terms are fitted with power-laws well.
The only exception is $\Var_x f_{d,\emptyset}^{(k)}(\xx)$ for the mean-field scaling: see the solid curve on the left plot.
Nevertheless, this term converges to a constant for large $d$, which indicates that our analysis becomes valid at least in the limit of large $d$.
Note that we have also matched numerical estimates of corresponding exponents with their theoretical values in Figure~1 of the main text.

\begin{figure*}[t]
    \centering
    \includegraphics[width=0.32\textwidth]{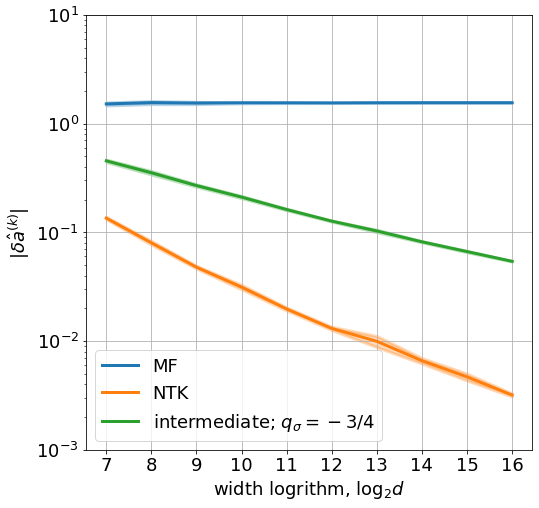}
    \includegraphics[width=0.32\textwidth]{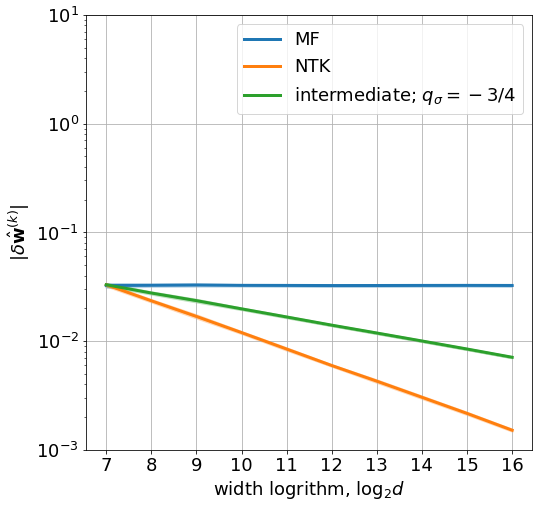}
    \includegraphics[width=0.32\textwidth]{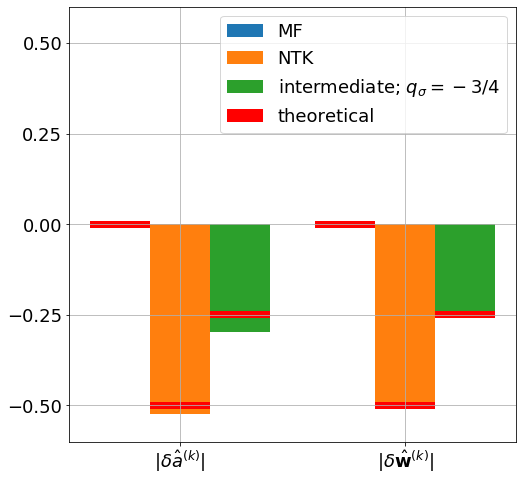}
    \caption{\textbf{Weight increments obey power-law dependencies with respect to the width.} \textit{Left:} absolute output weight increments averaged over hidden neurons as functions of width $d$. \textit{Center:} same for input weight increments. As one can see, weight increments are very well fitted with power-laws. \textit{Right:} numerical estimates for exponents of corresponding power-laws, as well as their theoretical values (denoted by red ticks). As one can see, theoretical values match numerical estimates very well. \textit{Setup:} We train a 1-hidden layer net on a subset of CIFAR2 (a dataset of first two classes of CIFAR10) of size 1000 with gradient descent. We take a reference net of width $d^* = 2^7 = 128$ trained with unscaled reference learning rates $\eta_a^* = \eta_w^* = 0.02$ and scale its hyperparameters according to MF (blue curves), NTK (orange curves) and intermediate scalings with $q_\sigma=-3/4$ (green curves, see main text). See SM~\ref{sec:appendix_experiments} for further details.}
    \label{fig:1hid_cifar2_sgd_da_dw_q}
\end{figure*}

\begin{figure*}[t]
    \centering
    \includegraphics[width=0.32\textwidth]{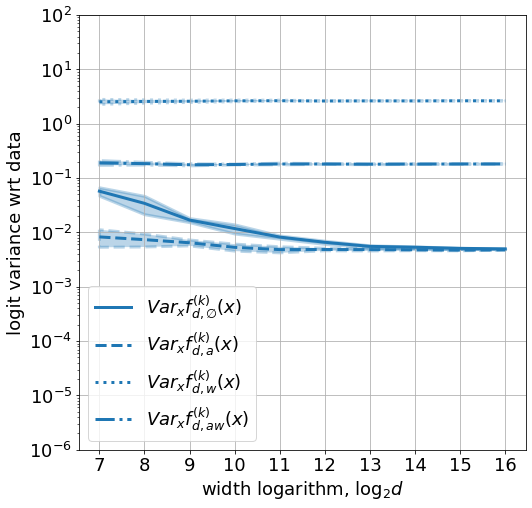}
    \includegraphics[width=0.32\textwidth]{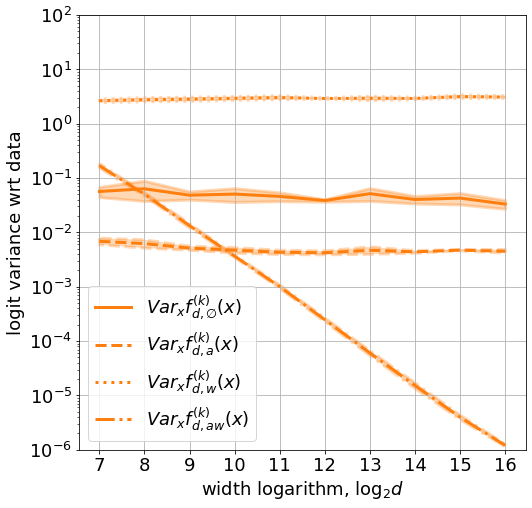}
    \includegraphics[width=0.32\textwidth]{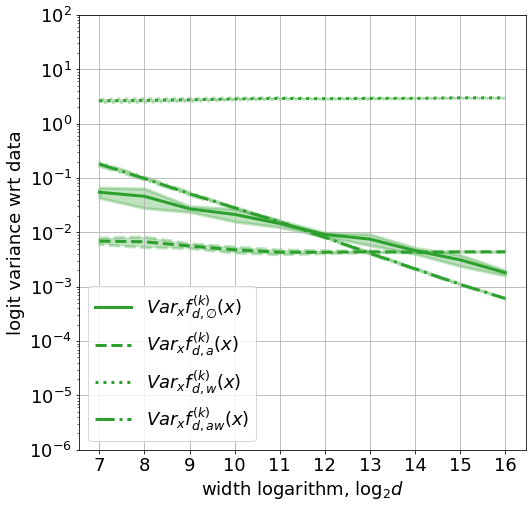}
    \caption{\textbf{Model decomposition terms obey power-law dependencies with respect to the width.} \textit{Left:} the variance with respect to the data distribution for terms of model decomposition (\ref{eq:f_decomposition_1hid_appendix}) as a function of width $d$ for the mean-field scaling. \textit{Center:} same for the NTK scaling. \textit{Right:} same for the intermediate scaling with $q_\sigma=-3/4$. As one can see, the data distribution variance of model decomposition terms are well-fitted with power-laws. \textit{Setup:} We train a 1-hidden layer net on a subset of CIFAR2 (a dataset of first two classes of CIFAR10) of size 1000 with a gradient descent. We take a reference net of width $d^* = 2^7 = 128$ trained with unscaled reference learning rates $\eta_a^* = \eta_w^* = 0.02$ and scale its hyperparameters according to MF (blue curves), NTK (orange curves) and intermediate scalings with $q_\sigma=-3/4$ (green curves, see text). See SM~\ref{sec:appendix_experiments} for further details.}
    \label{fig:1hid_cifar2_sgd_var_f_decomp}
\end{figure*}

\section{Derivation of $\varkappa$-terms in a one hidden layer case}
\label{sec:appendix_kappa_1hid}

For the sake of completeness, we copy all necessary definitions from Section~3 here.
A gradient descent step is defined as follows:
\begin{equation*}
    \Delta\delta\hat a_r^{(k)} = 
    - \hat\eta_a \sigma \EE \nabla_f^{(k)} \ell \; \phi((\hat\ww_r^{(0)} + \delta\hat\ww_r^{(k)})^T \xx),
\end{equation*}
\begin{equation}
    \Delta\delta\hat\ww_r^{(k)} = 
    - \hat\eta_w \sigma \EE \nabla_f^{(k)} \ell \; (\hat a_r^{(0)} + \delta\hat a_r^{(k)}) \phi'(\ldots) \xx,
    \label{eq:1hid_dynamics_appendix}
\end{equation}
\begin{equation*}
    \delta\hat a_r^{(0)} = 0, \; \delta\hat\ww_r^{(0)} = 0, \;
    \hat a_r^{(0)} \sim \NN(0, 1), \; \hat\ww_r^{(0)} \sim \NN(0, I);
\end{equation*}
\begin{equation*}
    f_d^{(k)}(\xx) = \sigma \sum_{r=1}^d (\hat a_r^{(0)} + \delta\hat a_r^{(k)}) \phi((\hat\ww_r^{(0)} + \delta\hat\ww_r^{(k)})^T \xx).
\end{equation*}
We assume:
\begin{equation*}
    \sigma \propto d^{q_\sigma}, \quad
    \hat\eta_{a/w} \propto d^{\tilde q_{a/w}}.
\end{equation*}
\begin{equation}
    |\delta\hat a_r^{(k)}| \propto d^{q_a^{(k)}}, \quad
    \|\delta\hat\ww_r^{(k)}\| \propto d^{q_w^{(k)}}.
\end{equation}
Assuming our model $f_d^{(k)}$ does not diverge with $d$, gradient step equations (\ref{eq:1hid_dynamics_appendix}) imply:
\begin{equation*}
    q_{a/w}^{(1)} = q_\sigma + \tilde q_{a/w},
\end{equation*}
\begin{equation}
    q_{a/w}^{(k+1)} = \max(q_{a/w}^{(k)}, q_{a/w}^{(1)} + \max(0, q_{w/a}^{(k)})).
    \label{eq:qk_1hid_appendix}
\end{equation}
We decompose our $f$ as:
\begin{equation}
    f_d^{(k)}(\xx) =
    f_{d,\emptyset}^{(k)}(\xx) + f_{d,a}^{(k)}(\xx) + f_{d,w}^{(k)}(\xx) + f_{d,aw}^{(k)}(\xx),
    \label{eq:f_decomposition_1hid_appendix}
\end{equation}
\begin{equation*}
    f_{d,\emptyset}^{(k)}(\xx) =
    \sigma \sum_{r=1}^d \hat a_r^{(0)} \phi'(\ldots) \hat\ww_r^{(0),T} \xx,
\end{equation*}
\begin{equation*}
    f_{d,a}^{(k)}(\xx) =
    \sigma \sum_{r=1}^d \delta\hat a_r^{(k)} \phi'(\ldots) \hat\ww_r^{(0),T} \xx,
\end{equation*}
\begin{equation*}
    f_{d,w}^{(k)}(\xx) =
    \sigma \sum_{r=1}^d \hat a_r^{(0)} \phi'(\ldots) \delta\hat\ww_r^{(k),T} \xx,
\end{equation*}
\begin{equation*}
    f_{d,aw}^{(k)}(\xx) =
    \sigma \sum_{r=1}^d \delta\hat a_r^{(k)} \phi'(\ldots) \delta\hat\ww_r^{(k),T} \xx,
\end{equation*}
where $\phi'(\ldots)$ is a shorthand for $\phi'((\hat\ww_r^{(0)} + \delta\hat\ww_r^{(k)})^T \xx)$ here.
We assume $f_d^{(k)}(\xx) \propto d^{q_f^{(k)}}$, $f_{d,\emptyset}^{(k)}(\xx) \propto d^{q_{f,\emptyset}^{(k)}}$, and so on.

By definition of decomposition~(\ref{eq:f_decomposition_1hid_appendix}) terms, we have:
\begin{equation*}
    q_{f,\emptyset}^{(k)} = q_\sigma + \varkappa_\emptyset^{(k)}, \quad
    q_{f,a/w}^{(k)} = q_{a/w}^{(k)} + q_\sigma + \varkappa_{a/w}^{(k)},
\end{equation*}
\begin{equation}
    q_{f,aw}^{(k)} = q_a^{(k)} + q_w^{(k)} + q_\sigma + \varkappa_{aw}^{(k)},
    \label{eq:q_terms_def_1hid_appendix}
\end{equation}
where all $\varkappa \in [1/2,1]$.

Our goal now is to compute $\varkappa$-terms for different values of $q_\sigma$ and $\tilde q_{a/w}$.
However it is more convenient to consider different cases for $q_a^{(1)}$ and $q_w^{(1)}$ instead.

\subsection{$q_a^{(1)}$ and $q_w^{(1)}$ are both negative}

In this case equations (\ref{eq:qk_1hid_appendix}) imply $q_{a/w}^{(k)} = q_{a/w}^{(1)} < 0$ $\forall k \geq 1$.
Hence $\phi'((\hat\ww_r^{(0)} + \delta\hat\ww_r^{(k)})^T \xx) \sim \phi'(\hat\ww_r^{(0),T} \xx)$ as $d \to \infty$.
Hence by the Central Limit Theorem, $\sum_{r=1}^d \hat a_r^{(0)} \phi'(\ldots) \hat\ww_r^{(0),T} \xx \propto d^{1/2}$.
This means that $\varkappa_\emptyset^{(k)} = 1/2$.

At the same time, using the definition of the gradient step for $\delta\hat\ww_r^{(k)}$,
\begin{eqnarray*}
    f_{d,w}^{(k)}(\xx)
    =
    \sigma \sum_{r=1}^d \hat a_r^{(0)} \phi'((\hat\ww_r^{(0)} + \delta\hat\ww_r^{(k)})^T \xx) \delta\hat\ww_r^{(k),T} \xx 
    \propto\\\propto
    \hat\eta_w \sigma^2 \sum_{r=1}^d \hat a_r^{(0)} \phi'((\hat\ww_r^{(0)} + \delta\hat\ww_r^{(k)})^T \xx) (\hat a_r^{(0)} + \delta\hat a_r^{(k-1)}) 
    \times\\\times 
    \phi'((\hat\ww_r^{(0)} + \delta\hat\ww_r^{(k-1)})^T \xx) \xx^T \xx 
    \sim\\\sim
    \hat\eta_w \sigma^2 \sum_{r=1}^d (\hat a_r^{(0)})^2 (\phi'(\hat\ww_r^{(0),T} \xx))^2 \xx^T \xx.
\end{eqnarray*}
We see that expression inside the sum has non-zero expectation, hence the sum scales as $d$, not as $d^{1/2}$.
Hence $\varkappa_w^{(k)} = 1$.
Using the similar reasoning we conclude that $\varkappa_a^{(k)} = 1$.
For $f_{d,aw}^{(k)}$ we have:
\begin{eqnarray*}
    f_{d,aw}^{(k)}(\xx)
    =
    \sigma \sum_{r=1}^d \delta\hat a_r^{(k)} \phi'((\hat\ww_r^{(0)} + \delta\hat\ww_r^{(k)})^T \xx) \delta\hat\ww_r^{(k),T} \xx 
    \propto\\\propto
    \hat\eta_a \hat\eta_w \sigma^3 \sum_{r=1}^d (\hat\ww_r^{(0)} + \delta\hat\ww_r^{(k-1)})^T \xx 
    \times\\\times 
    (\phi'((\hat\ww_r^{(0)} + \delta\hat\ww_r^{(k-1)})^T \xx))^2 
    \times\\\times 
    \phi'((\hat\ww_r^{(0)} + \delta\hat\ww_r^{(k)})^T \xx) (\hat a_r^{(0)} + \delta\hat a_r^{(k-1)}) \xx^T \xx
    \sim\\\sim
    \hat\eta_a \hat\eta_w \sigma^3 \sum_{r=1}^d \hat a_r^{(0)} \xx^T \xx (\phi'(\hat\ww_r^{(0),T} \xx))^3 \hat\ww_r^{(0),T} \xx.
\end{eqnarray*}
Here all random terms of the sum has zero expectation and $\hat a_r^{(0)}$ is independent from $(\phi'(\hat\ww_r^{(0),T} \xx))^3 \hat\ww_r^{(0)}$; hence the sum scales as $d^{1/2}$ and consequently $\varkappa_{aw}^{(k)} = 1/2$.

Summing up, if $q_{a/w}^{(1)} < 0$, then $\varkappa_\emptyset^{(k)} = \varkappa_{aw}^{(k)} = 1/2$ and $\varkappa_{a/w}^{(k)} = 1$ $\forall k \geq 1$.
Note that the NTK scaling is a subcase of this case.

\subsection{$q_a^{(1)}$ and $q_w^{(1)}$ are both zeros}

In this case equations (\ref{eq:qk_1hid_appendix}) imply $q_{a/w}^{(k)} = q_{a/w}^{(1)} = 0$ $\forall k \geq 1$.
Hence, generally, both $\delta\hat a^{(k)}$ and $\delta\hat\ww^{(k)}$ depend on both $\hat a^{(0)}$ and $\hat\ww^{(0)}$.
This implies that sums in definitions of $f_{d,a}^{(k)}$, $f_{d,w}^{(k)}$ and $f_{d,aw}^{(k)}$ scale as $d$; hence $\varkappa_a^{(k)} = \varkappa_w^{(k)} = \varkappa_{aw}^{(k)} = 1$ $\forall k > 1$.
Moreover, this implies that the sum
\begin{equation*}
    f_{d,\emptyset}^{(k)} =
    \sigma \sum_{r=1}^d \hat a_r^{(0)} \phi'((\hat\ww_r^{(0)} + \delta\hat\ww_r^{(k)})^T \xx) \hat\ww_r^{(0),T} \xx
\end{equation*}
also scales as $d$.
Hence $\varkappa_\emptyset^{(k)} = 1$ $\forall k \geq 1$.
Note that this is the case of the MF scaling.

\section{Other meaningful scalings}
\label{sec:appendix_other_scalings}

In the main text we have considered two solution classes for a system of equations and inequlaities that defines a meaningful scaling.
One class corresponds to the case of both $q_a^{(1)}$ and $q_w^{(1)}$ being less than zero, while the other one corresponds to the case of both of them being zeros.
In this section we consider all other possible cases.

\subsection{$q_a^{(1)} = 0$, while $q_w^{(1)} < 0$}
\label{sec:qa0_qw-}

In this case equations (\ref{eq:qk_1hid_appendix}) imply $q_a^{(k)} = q_a^{(1)} = 0$ and $q_w^{(k)} = q_w^{(1)} < 0$ $\forall k \geq 1$.
Since $\hat\ww^{(k)} \to \hat\ww^{(0)}$, $\delta\hat a^{(k)}$ does not become independent on $\hat\ww^{(0)}$ as $d\to\infty$; hence $\varkappa_a^{(k)} = 1$.
Also, since $q_w^{(k)} < 0$, $\phi'(\hat\ww^{(k),T} \xx) \to \phi'(\hat\ww^{(0),T} \xx)$; hence $\varkappa_\emptyset^{(k)} = 1/2$.

A condition $q_{f,a}^{(k)} = q_\sigma + q_a^{(1)} + \varkappa_a^{(k)} \leq 0$ then implies that $q_\sigma \leq -1$.
Hence $q_{f,\emptyset}^{(k)} = q_\sigma + \varkappa_\emptyset^{(k)} \leq -1/2 < 0$. 
Moreover, $q_{f,w}^{(k)} = q_\sigma + q_w^{(k)} + \varkappa_w^{(k)} < 0$, since $\varkappa_w^{(k)} \leq 1$, and similarly, $q_{f,aw}^{(k)} = q_\sigma + q_a^{(k)} + q_w^{(k)} + \varkappa_{aw}^{(k)} < 0$ since $\varkappa_{aw}^{(k)} \leq 1$.

Hence in order to have a non-vanishing limit model we have to have $q_{f,a}^{(k)} = 0$ which implies $q_\sigma = -1$.
Note that $q_a^{(1)} = q_\sigma + \tilde q_a = 0$; since then $\tilde q_a = 1$.
Since $f_{d,a}^{(k)}$ is the only term of the model decomposition that remains finite as $d \to \infty$, we essentialy learn the output layer only in the limit of $d \to \infty$.
Hence we can describe the dynamics of the limit model both in terms of the evolution of the limit measure and in terms of a constant deterministic limit kernel.

Indeed, suppose $\sigma = \sigma^* d^{-1}$ and $\hat\eta_a = \hat\eta_a^* d$. 
The limit measure evolution writes as follows:
\begin{equation*}
    f_\infty^{(k)}(\xx) = \sigma^* \int \hat a \phi(\hat\ww^{T} \xx) \, \mu_\infty^{(k)}(d\hat a, d\hat\ww);
\end{equation*}
\begin{equation*}
    \mu_\infty^{(k+1)} = \TT_a(\mu_\infty^{(k)}; \; \hat\eta_a^* \sigma^*, \sigma^*),
    \qquad
    \mu_\infty^{(0)} = \NN_{1+d_0}(0, I),
\end{equation*}
where a gradient descent step operator $\TT_a$ is defined on probabilistic measures $\mu$ supported on a finite set of atoms $d$ as follows:
\begin{equation*}
    \TT_a(\mu_d; \; \hat\eta_a^* \sigma^*, \sigma^*) =
    \frac{1}{d} \sum_{r=1}^d \delta_{\hat a'_r} \otimes \delta_{\hat\ww_r},
\end{equation*}
where
\begin{equation*}
    \hat a'_r =
    \hat a_r - \hat\eta_a^* \sigma^* \EE_{\xx,y} \left. \frac{\partial \ell(y,z)}{\partial z} \right|_{z=f_d(\xx; \sigma^*)} \phi(\hat\ww_r^T \xx),
\end{equation*}
and $f_d(\xx; \sigma^*) = \sigma^* \int \hat a \phi(\hat\ww^{T} \xx) \, \mu_d(d\hat a, d\hat\ww)$ for $(\hat a_r, \hat\ww_r)$, $r \in [d]$, being atoms of measure $\mu_d$.

Consider now a kernel $\tilde\Theta_{a,\infty}$ defined as follows:
\begin{equation*}
    \tilde\Theta_{a,\infty}(\xx,\xx') =
    \hat\eta_a^* \sigma^{*,2} \EE_{\hat\ww \sim \NN(0,I_{d_0})} \phi(\hat\ww^T \xx) \phi(\hat\ww^T \xx').
\end{equation*}
Using the same argument as in SM~\ref{sec:appendix_limiting_dynamics}, we can write a continuous-time evolution of the limit model in terms of this kernel:
\begin{equation*}
    \dot f_\infty^{(t)}(\xx') = 
    -\EE_{\xx,y} \left. \frac{\partial \ell(y,z)}{\partial z} \right|_{z=f_\infty^{(t)}(\xx)} \tilde\Theta_{a,\infty}(\xx,\xx'),
\end{equation*}
\begin{equation*}
    f_\infty^{(0)}(\xx) =
    \EE_{(\hat a, \hat\ww) \sim \NN(0,I_{1+d_0})} \hat a \phi(\hat\ww^{T} \xx) =
    0.
\end{equation*}
Moreover, for the same argument as in SM~\ref{sec:appendix_limiting_dynamics}, the similar evolution equation holds also for the discrete-time evolution:
\begin{equation*}
    \Delta f_\infty^{(k)}(\xx') = 
    -\EE_{\xx,y} \left. \frac{\partial \ell(y,z)}{\partial z} \right|_{z=f_\infty^{(k)}(\xx)} \tilde\Theta_{a,\infty}(\xx,\xx').
\end{equation*}

\subsection{$q_w^{(1)} = 0$, while $q_a^{(1)} < 0$}
\label{sec:qa-_qw0}

This case is almost analogous to the previous one.
Equations (\ref{eq:qk_1hid_appendix}) imply $q_w^{(k)} = q_w^{(1)} = 0$ and $q_a^{(k)} = q_a^{(1)} < 0$ $\forall k \geq 1$, and $\delta\hat w^{(k)}$ does not become independent on $\hat a^{(0)}$ as $d\to\infty$; hence $\varkappa_w^{(k)} = 1$.
Note that in contrast to the previous case, since $q_w^{(k)} = 0$, $\phi'(\hat\ww^{(k),T} \xx) \nrightarrow \phi'(\hat\ww^{(0),T} \xx)$; hence $\varkappa_\emptyset^{(k)} = 1$.

A condition $q_{f,w}^{(k)} = q_\sigma + q_w^{(1)} + \varkappa_w^{(k)} \leq 0$ (or, equivalently, a condition $q_{f,\emptyset}^{(k)} = q_\sigma + \varkappa_\emptyset^{(k)} \leq 0$) then implies that $q_\sigma \leq -1$.
Hence $q_{f,a}^{(k)} = q_\sigma + q_a^{(k)} + \varkappa_a^{(k)} < 0$, since $\varkappa_a^{(k)} \leq 1$, and similarly, $q_{f,aw}^{(k)} = q_\sigma + q_a^{(k)} + q_w^{(k)} + \varkappa_{aw}^{(k)} < 0$, since $\varkappa_{aw}^{(k)} \leq 1$.

Hence in order to have a non-vanishing limit model we have to have $q_{f,\emptyset}^{(k)} = q_{f,w}^{(k)} = 0$, which implies $q_\sigma = -1$.
Note that $q_w^{(1)} = q_\sigma + \tilde q_w = 0$; since then $\tilde q_w = 1$.
In this case we again essentialy learn only a single layer in the limit of $d \to \infty$.
However a kernel which drives the dynamics evolves with $k$ since $w^{(k)} \nrightarrow w^{(0)}$:
\begin{multline*}
    \tilde\Theta^{(k)}_{w,\infty}(\xx,\xx') =
    \hat\eta_w^* \sigma^{*,2} \lim_{d\to\infty} \frac{1}{d} \sum_{d=1}^\infty \EE_{\hat a \sim \NN(0,1)} | \hat a |^2 \times\\\times \phi'(\hat\ww^{(k),T} \xx) \phi'(\hat\ww^{(k),T} \xx') \xx^T \xx'.
\end{multline*}
Nevertheless, the dynamics can be described in terms of the measure evolution similar to the previous case.

\subsection{$q_a^{(1)} > 0$, while $q_a^{(1)} + q_w^{(1)} \leq 0$}
\label{sec:qa+_qw-}

In this case equations (\ref{eq:qk_1hid_appendix}) imply $q_a^{(k)} = q_a^{(1)} > 0$, while $q_w^{(k)} = q_a^{(1)} + q_w^{(1)} \leq 0$ $\forall k > 1$.
Similar to the case of SM~\ref{sec:qa0_qw-}, $\delta\hat a^{(k)}$ does not become independent on $\hat\ww^{(0)}$ as $d\to\infty$; hence $\varkappa_a^{(k)} = 1$.

Consider $k > 1$.
A condition $q_{f,a}^{(k)} = q_\sigma + q_a^{(1)} + \varkappa_a^{(k)} \leq 0$ then implies $q_\sigma \leq -1 - q_a^{(1)} < -1$.
Hence $q_{f,\emptyset}^{(k)} = q_\sigma + \varkappa_\emptyset^{(k)} < 0$ since $\varkappa_\emptyset^{(k)} \leq 1$.
Moreover, $q_{f,w}^{(k)} = q_\sigma + q_w^{(k)} + \varkappa_w^{(k)} < 0$ since $\varkappa_w^{(k)} \leq 1$ and $q_w^{(k)} = q_a^{(1)} + q_w^{(1)} \leq 0$, and similarly, $q_{f,aw}^{(k)} = q_\sigma + q_a^{(k)} + q_w^{(k)} + \varkappa_{aw}^{(k)} \leq q_{f,a}^{(k)} \leq 0$ since $\varkappa_{aw}^{(k)} \leq 1$.

Hence in order to have a non-vanishing limit model we have to have $q_{f,a}^{(k)} = 0$, which implies $q_a^{(1)} = q_\sigma + \tilde q_a = -1-q_\sigma$.
Since then $\tilde q_a = -1-2q_\sigma$, while $q_\sigma < -1$.
Suppose $q_w^{(k)} = q_a^{(1)} + q_w^{(1)} < 0$.
In this case $q_{f,aw}^{(k)} < 0$, hence $f_{d,a}^{(k)}$ is the only term of the model decomposition that remains finite as $d \to \infty$, and we learn the output layer only in the limit of $d \to \infty$, as was the case of SM~\ref{sec:qa0_qw-}.
In this case we are again able to describe the dynamics of the limit model both in terms of the evolution of the limit measure and in terms of a constant deterministic limiting kernel.

While the kernel description does not change at all compared to the case described in SM~\ref{sec:qa0_qw-}, measure evolution equations require slight modifications.
Indeed, suppose $\sigma = \sigma^* d^{q_\sigma}$ and $\hat\eta_a = \hat\eta_a^* d^{-1-2q_\sigma}$. 
The limit measure evolution writes as follows:
\begin{equation*}
    f_\infty^{(k)}(\xx) = \sigma^* \int \hat a \phi(\hat\ww^{T} \xx) \, \mu_\infty^{(k)}(d\hat a, d\hat\ww);
\end{equation*}
\begin{equation*}
    \mu_\infty^{(k+1)} = \TT_a(\mu_\infty^{(k)}; \; \hat\eta_a^* \sigma^*, \sigma^*),
    \qquad
    \mu_\infty^{(0)} = \delta \otimes \NN(0, I_{d_0}),
\end{equation*}
where a gradient descent step operator $\TT_a$ is defined on probabilistic measures $\mu$ supported on a finite set of atoms $d$ as follows:
\begin{equation*}
    \TT_a(\mu_d; \; \hat\eta_a^* \sigma^*, \sigma^*) =
    \frac{1}{d} \sum_{r=1}^d \delta_{\hat a'_r} \otimes \delta_{\hat\ww_r},
\end{equation*}
where
\begin{equation*}
    \hat a'_r =
    \hat a_r - \hat\eta_a^* \sigma^* \EE_{\xx,y} \left. \frac{\partial \ell(y,z)}{\partial z} \right|_{z=f_d(\xx; \sigma^*)} \phi(\hat\ww_r^T \xx),
\end{equation*}
and $f_d(\xx; \sigma^*) = \sigma^* \int \hat a \phi(\hat\ww^{T} \xx) \, \mu_d(d\hat a, d\hat\ww)$ for $(\hat a_r, \hat\ww_r)$, $r \in [d]$, being atoms of measure $\mu_d$.

The only thing changed here is that in the limit output weights $\hat a$ are initialized with zeros.
The reason for this is that the increment at the first step $\delta\hat a^{(0)} = -\hat\eta_a \sigma \EE \nabla_f^{(0)} \ell \, \phi(\hat\ww^{(0),T} \xx)$ grows as $d^{-1-q_\sigma}$ as $d \to \infty$.
Hence $\hat a^{(k)} \to \delta\hat a^{(k)}$ as $d\to\infty$ for $k \geq 1$.

Suppose now $q_w^{(k)} = q_a^{(1)} + q_w^{(1)} = 0$.
In this case $\delta\hat a^{(k)}$ and $\delta\hat\ww^{(k)}$ do not become independent, since $\hat\ww^{(k)} \nrightarrow \hat\ww^{(0)}$; hence $\varkappa_{aw}^{(k)} = 1$. 
This implies that $q_{f,aw}^{(k)} = q_\sigma + q_a^{(k)} + q_w^{(k)} + \varkappa_{aw}^{(k)} = 0$ for $k > 1$, hence two terms of the model decomposition remain finite as $d \to \infty$: $f_{d,a}^{(k)}$ and $f_{d,aw}^{(k)}$.
Note that $q_a^{(1)} + q_w^{(1)} = 0$ implies $\tilde q_w = -\tilde q_a - 2q_\sigma = 1$.

Suppose $\hat\eta_w = \hat\eta_w^* d$.
In this case we are again able to describe the dynamics of the limit model in terms of the evolution of the limit measure:
\begin{equation*}
    f_\infty^{(k)}(\xx) = \sigma^* \int \hat a \phi(\hat\ww^{T} \xx) \, \mu_\infty^{(k)}(d\hat a, d\hat\ww);
\end{equation*}
\begin{equation*}
    \mu_\infty^{(k+1)} = \TT_a(\mu_\infty^{(k)}; \; \hat\eta_a^* \sigma^*, \sigma^*),
    \qquad
    \mu_\infty^{(0)} = \delta \otimes \NN(0, I_{d_0}),
\end{equation*}
where a gradient descent step operator $\TT_a$ is defined on probabilistic measures $\mu$ supported on a finite set of atoms $d$ as follows:
\begin{equation*}
    \TT(\mu_d; \; \hat\eta_a^* \sigma^*, \hat\eta_w^* \sigma^*, \sigma^*) =
    \frac{1}{d} \sum_{r=1}^d \delta_{\hat a'_r} \otimes \delta_{\hat\ww_r},
\end{equation*}
where
\begin{equation*}
    \hat a'_r =
    \hat a_r - \hat\eta_a^* \sigma^* \EE_{\xx,y} \left. \frac{\partial \ell(y,z)}{\partial z} \right|_{z=f_d(\xx; \sigma^*)} \phi(\hat\ww_r^T \xx),
\end{equation*}
\begin{equation*}
    \hat\ww'_r =
    \hat\ww_r - \hat\eta_w^* \sigma^* \EE_{\xx,y} \left. \frac{\partial \ell(y,z)}{\partial z} \right|_{z=f_d(\xx; \sigma^*)} \hat a_r \phi'(\hat\ww_r^T \xx),
\end{equation*}
and $f_d(\xx; \sigma^*) = \sigma^* \int \hat a \phi(\hat\ww^{T} \xx) \, \mu_d(d\hat a, d\hat\ww)$ for $(\hat a_r, \hat\ww_r)$, $r \in [d]$, being atoms of measure $\mu_d$.

We have a zero initialization for the output weights for the same reason as for the case of $q_w^{(k)} < 0$.
Note that in contrast to the above-mentioned case, the case of $q_w^{(k)} = 0$ cannot be described in terms of a constant limit kernel.
Indeed, we have a stochastic time-dependent kernel for finite width $d$ associated with output weights learning:
\begin{equation*}
    \tilde\Theta_{a,\infty}^{(k)}(\xx,\xx') =
    \hat\eta_a^* \sigma^{*,2} \frac{1}{d} \sum_{r=1}^d \phi(\hat\ww_r^{(k),T} \xx) \phi(\hat\ww_r^{(k),T} \xx').
\end{equation*}
This kernel converges to a deterministic one as $d\to\infty$ by the Law of Large Numbers, however, the limit kernel stays step-dependent, since $\hat\ww^{(k)} = \hat\ww^{(0)} + \delta\hat\ww^{(k)}$, while the last term here does not vanish as $d\to\infty$.

Note that the "default" case we have considered in the main text falls into the current case.
Indeed, by default we have $\sigma \propto d^{-1/2}$ and $\eta_{a/w} \propto 1$. 
This implies $q_\sigma = -1/2$, $\tilde q_a = 1$ and $\tilde q_w = 0$; consequently, $q_a^{(1)} = 1/2$ and $q_w^{(1)} = -1/2$.
However, as we have shown above, having $q_\sigma \leq -1 - q_a^{(1)} = -3/2$ is necessary to guarantee that the limit model does not diverge.
As we observe in Figure~1 of the main text a limit model resulted from the default scaling indeed diverges.

\subsection{$q_w^{(1)} > 0$, while $q_a^{(1)} + q_w^{(1)} \leq 0$}
\label{sec:qa-_qw+}

The difference between this case and the previous one is essentially the same as between cases of SM~\ref{sec:qa-_qw0} and of SM~\ref{sec:qa0_qw-}.
For this reason we leave this case as an exercise for the reader.

\subsection{$q_a^{(1)} + q_w^{(1)} > 0$}
\label{sec:qa+_qw+}

Suppose first that $q_a^{(1)} > 0$.
In this case equations (\ref{eq:qk_1hid_appendix}) imply $q_w^{(2)} = q_a^{(1)} + q_w^{(1)} > 0$ and $q_a^{(2)} \geq q_a^{(1)} > 0$.
It is easy to see that equations \ref{eq:qk_1hid_appendix} further imply $q_a^{(2k)} = q_w^{(2k)} = k (q_a^{(2)} + q_w^{(2)})$ $\forall k \geq 1$.
This means that $q_a^{(k)}$ and $q_w^{(k)}$ grow linearly with $k$.
Hence all of $q_{f,a}^{(k)}$, $q_{f,w}^{(k)}$, $q_{f,aw}^{(k)}$ become positive for large enough $k$ irrespective of $q_\sigma$.

Obviously, the same holds if $q_w^{(1)} > 0$.
Hence in this case our analysis suggests that a limit model $f_{\infty}^{(k)}$ diverges with $d$ for large enough $k$.
However, when our analysis predicts that a limit model diverges, we cannot guarantee that $\nabla_f^{(k)} \ell$ does not vanish with $d$, hence equations \ref{eq:qk_1hid_appendix} become generally incorrect.
Indeed, if a model reaches $100\%$ train accuracy at step $k$, then $\nabla_f^{(k)} \ell$ vanishes exponentially if $f$ grows.
This means that $f$ cannot really diverge width $d$ if it reaches $100\%$ train accuracy.

\section{A discrete-time mean-field limit of a network with a single hidden layer}
\label{sec:appendix_discrete_mf}

In this section we omit "hats" for brevity, assuming all relevant quantities to be scaled appropriately.

Recall that in the MF limit $\sigma \propto d^{-1}$ and $\eta_{a/w} \propto d$.
Suppose $\sigma = \sigma^* d^{-1}$ and w.l.o.g. $\eta_{a/w} = \eta^* d$.

We closely follow the measure-theoretic formalism of \citet{sirignano2018lln}.
Consider a measure in $( a, \ww)$-space at each step $k$ for a given $d$:
$$
    \mu_d^{(k)} = \frac{1}{d} \sum_{r=1}^d \delta_{ a_r^{(k)}} \otimes \delta_{\ww_r^{(k)}}.
$$
Given this, a neural network output can be represented as follows:
$$
    f_d^{(k)}(\xx) =
    \sigma^* \int  a \phi(\ww^{T} \xx) \mu_d^{(k)}(da, \, d\ww).
$$

A gradient descent step is written as follows:
\begin{equation*}
    \Delta a_r^{(k)} = 
    - \eta^* \sigma^* \EE_{\xx,y} \nabla_f^{(k)} \ell \; \phi(\ww_r^{(k),T} \xx),
\end{equation*}
\begin{equation}
    \Delta\ww_r^{(k)} = 
    - \eta^* \sigma^* \EE_{\xx,y} \nabla_f^{(k)} \ell \;  a_r^{(k)} \phi'(\ww_r^{(k),T} \xx) \xx.
    \label{eq:1hid_dynamics_original}
\end{equation}

For technical reasons we assume weights $a_r$ and $\ww_r$ $\forall r \in [d]$ to be initialized from the distribution $\PP$ with compact support:
\begin{equation}
     a_r^{(0)} \sim \PP, \quad w_{r,j}^{(0)} \sim \PP \quad \forall r \in [d] \; \forall j \in [d_0].
    \label{eq:init_conds}
\end{equation}
One can notice that in the main body of this work we have assumed $\PP$ to be $\NN(0,1)$ that does not have a compact support.
Nevertheless, it is more common in practice to use a truncated normal distribution instead of the original normal one, which was used in the main body for the ease of explanation only.

We introduce a transition operator $\TT$ which represents a gradient descent step (\ref{eq:1hid_dynamics_original}):
\begin{equation}
    \mu_d^{(k+1)} =
    \TT(\mu_d^{(k)}; \; \eta^*, \sigma^*).
    \label{eq:mu_dynamics}
\end{equation}
This operator depends explicitly on $\sigma^*$ because $\nabla_f^{(k)} \ell$ is a gradient of $f_d^{(k)}$ and the latter depends on $\sigma^*$.
This representation clearly shows that a gradient descent defines a Markov chain for measures on the weight space with deterministic transitions.
The initial measure $\mu_d^{(0)}$ is given by initial conditions~(\ref{eq:init_conds}).
Since they are random, measure $\mu_d^{(k)}$ is a random measure for any $k \geq 0$ and for any $d \in \mathbb{N}$.
Nevertheless, for all $k \geq 0$ $\mu_d^{(k)}$ converges to a corresponding limit measure as the following theorem states:
\begin{thm}
    Suppose $\ell(y,\cdot) \in C^2(\RR)$, $\partial \ell(y,z) / \partial z$ is bounded and Lipschitz continuous and $\phi$ is Lipschitz continuous.
    Suppose also that $\xx$ has finite moments up to the fourth one.
    Finally, assume that the distribution of initial weights $\PP$ has compact support.
    Then $\forall k \geq 0$ there exists a measure $\mu_\infty^{(k)}$ such that $\mu_d^{(k)}$ converges to $\mu_\infty^{(k)}$ weakly as $d \to \infty$ wrt to the 2-Wasserstein metric and each measure $\mu_d^{(k)}$ is supported on a ball $\BB_{R_k}$ a.s. for all $d$.
    \label{thm:asymptotic_similarity_1hid_relu_gd}
\end{thm}
\begin{proof}
    We prove this by induction on $k$.

    Let $k = 0$.
    Any measure $\mu$ on the weight space is uniquely determined by its action on all $g \in C(\RR^{1+d_0})$ with compact support: $\langle g, \mu \rangle = \int g(a, \ww) \mu(da, \, d\ww)$.
    If this measure is random, then the last integral is a random variable.
    Hence $\mu_d^{(0)}$ converges to $\mu_\infty^{(0)} = \PP$ weakly as $d \to \infty$, iff for all $g \in C(\RR^{1+d_0})$ with compact support $\langle g, \mu_d^{(0)} \rangle$ converges to $\langle g, \mu_\infty^{(0)} \rangle$ weakly as $d \to \infty$.

    Let $h \in C_b(\RR)$.
    Consider
    \begin{eqnarray*}
        \lim_{d \to \infty} \EE_{\aa^{(0)},  W^{(0)}} h\left(\left\langle g, \mu_d^{(0)} \right\rangle\right) =\\=
        \lim_{d \to \infty} \EE_{\aa^{(0)},  W^{(0)}} h\left(\frac{1}{d} \sum_{r=1}^d g( a^{(0)}_r, \ww^{(0)}_r)\right) =\\=
        h\left(\EE_{ a^{(0)}, \ww^{(0)}} g\left( a^{(0)}, \ww^{(0)}\right)\right) =
        h\left(\left\langle g, \mu_\infty^{(0)} \right\rangle\right),
    \end{eqnarray*}
    where the second equality comes from the Law of Large Numbers which is valid since initial weights are i.i.d.
    This proves a weak convergence of $\langle g, \mu_d^{(0)} \rangle$ to $\langle g, \mu_\infty^{(0)} \rangle$.
    As was noted above, this is equivalent to a weak convergence of measures $\mu_d^{(0)}$:
    $$
    \lim_{d \to \infty} \EE_{\aa^{(0)},  W^{(0)}} h[\mu_d^{(0)}] =
    h[\mu_\infty^{(0)}]
    $$
    for any $h \in C_b(\MM(\RR^{1+d_0}))$.

    Also, since all $a_r \sim \PP$, $w_{r,j} \sim \PP$ and $\PP$ has compact support, $\mu_d^{(0)}$ has compact support almost surely.
    Hence we can write $\mu_d^{(0)} \in \MM(\BB_{R_0}^{1+d_0})$ a.s. for some $R_0 < \infty$ $\forall d$.

    We have proven the induction base.
    By induction assumption, we have $\mu_d^{(k)} \in \MM(\BB_{R_k}^{1+d_0})$ a.s. for some $R_k < \infty$ $\forall d$.
    Let for any $h \in C_b(\MM(\RR^{1+d_0}))$
    $$
    \lim_{d \to \infty} \EE_{\aa^{(0)},  W^{(0)}} h[\mu_d^{(k)}] =
    h[\mu_\infty^{(k)}].
    $$
    By definition, this means weak convergence of measures $\mu_d^{(k)}$ to $\mu_\infty^{(k)}$.
    Then we have:
    $$
    \lim_{d \to \infty} \EE_{\aa^{(0)},  W^{(0)}} h[\mu_d^{(k+1)}] =
    \lim_{d \to \infty} \EE_{\aa^{(0)},  W^{(0)}} h[\TT(\mu_d^{(k)})].
    $$
    In order to prove that this limit exists and equals to $h[\TT(\mu_\infty^{(k)})]$ we have to show that $h \circ \TT \in C_b(\MM(\BB_{R_k}^{1+d_0}))$.

    We prove the following lemma in Section~\ref{sec:appendix_T_in_C}:
    \begin{lemma}
        \label{lemma:transition_operator}
        Given conditions of Theorem~\ref{thm:asymptotic_similarity_1hid_relu_gd} and $R < \infty$, the transition operator $T$ that performs a gradient descent step~(\ref{eq:mu_dynamics}) is continuous wrt the 2-Wasserstein metric on $\MM(\BB_R^{1+d_0})$.
    \end{lemma}    

    Hence $h \circ \TT \in C_b(\MM(\BB_{R_k}^{1+d_0}))$.
    Since then, by the induction hypothesis for all $h \in C_b(\MM(\RR^{1+d_0}))$
    \begin{eqnarray*}
        \lim_{d \to \infty} \EE_{\aa^{(0)},  W^{(0)}} h[\mu_d^{(k+1)}] 
        =\\=
        \lim_{d \to \infty} \EE_{\aa^{(0)},  W^{(0)}} h[\TT(\mu_d^{(k)})] 
        =
        h[\TT(\mu_\infty^{(k)})].
    \end{eqnarray*}
    We then define $\mu_\infty^{(k+1)} = \TT(\mu_\infty^{(k)})$.

    Also, it easy to see that since $\phi$, $\phi'$ and $\partial\ell(y,z) / \partial z$ are bounded and the distribution of $\xx$ has a bounded variation, $\mu_d^{(k)} \in \MM(\BB_{R_k}^{1+d_0})$ a.s. implies $\mu_d^{(k+1)} = \TT \mu_d^{(k)} \in \MM(\BB_{R_{k+1}}^{1+d_0})$ a.s. for some $R_{k+1} < \infty$.
    
    We have proven that for all $k \geq 0$ $\mu_d^{(k)}$ converges to $\mu_\infty^{(k)}$ weakly as $d \to \infty$ wrt the 2-Wasserstein metric and $\mu_d^{(k)}$ has compact support a.s. for any $d \in \mathbb{N}$.
\end{proof}

\begin{corollary}[Theorem~1 of Section~3, restated]
    \label{thm:discrete_mf}
    Given the same conditions as in Theorem~\ref{thm:asymptotic_similarity_1hid_relu_gd}, following statements hold:
    \begin{enumerate}
        \item $\forall k \geq 0$ $\mu_d^{(k)}$ converges to $\mu_\infty^{(k)}$ in probability as $d \to \infty$;
        \item $f_d^{(k)}(\xx)$ converges to some $f_\infty^{(k)}(\xx)$ in probability as $d \to \infty$ $\forall \xx \in \XX$.
    \end{enumerate}
\end{corollary}
\begin{proof}
    Since weak convergence to a constant implies convergence in probability, the first statement directly follows from Theorem~\ref{thm:asymptotic_similarity_1hid_relu_gd}.

    By definition, weak convergence of $\mu_d^{(k)}$ means for any $h \in C_b(\MM(\RR^{1+d_0}))$
    $$
    \lim_{d \to \infty} \EE_{\aa^{(0)},  W^{(0)}} h[\mu_d^{(k)}] =
    h[\mu_\infty^{(k)}].
    $$
    Hence for any $g \in C_b(\RR)$
    \begin{eqnarray*}
        \lim_{d \to \infty} \EE_{\aa^{(0)},  W^{(0)}} g(f_d^{(k)}(\xx)) 
        =\\=
        \lim_{d \to \infty} \EE_{\aa^{(0)},  W^{(0)}} g(f[\mu_d^{(k)}; \xx]) 
        =\\=
        \lim_{d \to \infty} \EE_{\aa^{(0)},  W^{(0)}} (g \circ f)[\mu_d^{(k)}; \xx] =
        (g \circ f)[\mu_\infty^{(k)}; \xx],            
    \end{eqnarray*}
    since $f[\cdot; \xx] \in C(\MM(\RR^{1+d_0}))$ for any $\xx \in \XX$.

    Hence $f_d^{(k)}(\xx) = f[\mu_d^{(k)}; \xx]$ converges weakly to $f_\infty^{(k)}(\xx) = f[\mu_\infty^{(k)}; \xx]$ as $d \to \infty$.
    By the same argument as above, this implies convergence in probability.
\end{proof}

\subsection{A gradient descent step defines a continuous operator in the space of weight-space measures}
\label{sec:appendix_T_in_C}

\begin{proof}[Proof of Lemma \ref{lemma:transition_operator}]
    Without loss of generality assume $\sigma^* = \eta^* = 1$.
    Consider a sequence of measures $\mu_d \in \MM(\BB_R^{1+d_0})$ that converges to $\mu_\infty \in \MM(\BB_R^{1+d_0})$ wrt the 2-Wasserstein metric.
    We have to prove that $\TT \mu_d$ converges to $\TT \mu_\infty$ wrt the 2-Wasserstein metric.

    Define $\theta_d = (a_d, \ww_d) \in \BB_R^{1+d_0}$ and $\delta\theta_d = \theta_\infty - \theta_d = (a_\infty - a_d, \ww_\infty - \ww_d) \in \BB_R^{1+d_0}$.
    For a given $d$ consider a sequence of measures $\mu_{d,\infty}^j \in \MM(\BB_R^{1+d_0} \otimes \BB_R^{1+d_0})$ with marginals equal to $\mu_d$ and $\mu_\infty$ respectively, as required by the definition of the Wasserstein metric.
    Choose a sequence in such a way that
    \begin{multline*}
        \lim_{j \to \infty} \int (\|\delta\theta_d\|_2^2 \, \mu^j_{d,\infty} (d\theta_d, \, d\theta_\infty) 
        =\\=
        \inf_{\mu_{d,\infty}} \int (\|\delta\theta_d\|_2^2 \, \mu_{d,\infty} (d\theta_d, \, d\theta_\infty) =
        \WW_2^2(\mu_d, \mu_\infty),
    \end{multline*}
    where infium is taken over all $\mu_{d,\infty} \in \MM(\BB_R^{1+d_0} \otimes \BB_R^{1+d_0})$ with marginals equal to $\mu_d$ and $\mu_\infty$ respectively as required by the definition of the Wasserstein metric.
    A sequence $\{\mu_{d,\infty}^j\}_{j=1}^\infty$ exists by properties of infium.
    Then we have the following:
    \begin{eqnarray*}
        \WW_2^2(\TT\mu_d, \TT\mu_\infty) 
        \leq\\\leq
        \lim_{j \to \infty} \int
            \|\delta\theta_d + \delta\Delta\theta_d\|_2^2
        \, \mu^j_{d,\infty} (d\theta_d, \, d\theta_\infty),
    \end{eqnarray*}
    where we have defined
    \begin{equation*}
        \Delta\theta_d = 
        \Bigl( -\EE \nabla_{f_d}\ell \, \phi(\ww_d^{T} \xx), -\EE \nabla_{f_d}\ell \, a_d \phi'(\ww_d^{T} \xx) \xx \Bigr),
    \end{equation*}
    \begin{equation*}
        \nabla_{f_d}\ell = \left.\frac{\partial \ell(y,z)}{\partial z} \right|_{z=f[\mu_d; \xx]}
    \end{equation*}
    and $\delta\Delta\theta_d = \Delta\theta_\infty - \Delta\theta_d$ respectively.
    From this follows:
    \begin{eqnarray*}
        \WW_2^2(\TT\mu_d, \TT\mu_\infty) 
        \leq
        \lim_{j \to \infty} \int
            \|\delta\theta_d\|_2^2
        \, \mu^j_{d,\infty} (d\theta_d, \, d\theta_\infty) 
        +\\+
        \lim_{j \to \infty} \int
            \|\delta\Delta\theta_d\|_2^2
        \, \mu^j_{d,\infty} (d\theta_d, \, d\theta_\infty) 
        +\\+
        2 \lim_{j \to \infty} \int
            \langle \delta\theta_d, \delta\Delta\theta_d \rangle
        \, \mu^j_{d,\infty} (d\theta_d, \, d\theta_\infty).
    \end{eqnarray*}
    Consequently,
    \begin{eqnarray}
        \WW_2^2(\TT\mu_d, \TT\mu_\infty) 
        \leq
        \WW_2^2(\mu_d, \mu_\infty) 
        +\\+
        \lim_{j \to \infty} \int
            \|\delta\Delta\theta_d\|_2^2
        \, \mu^j_{d,\infty} (d\theta_d, \, d\theta_\infty) 
        +\\+
        4 R \lim_{j \to \infty} \sqrt{\int
            \|\delta\Delta\theta_d\|_2^2
        \, \mu^j_{d,\infty} (d\theta_d, \, d\theta_\infty)}.
        \label{eq:W_2_tau_mu_upper_bound}
    \end{eqnarray}
    The last term comes (1) from the Cauchy-Schwartz inequality: $\langle \delta\theta_d, \delta\Delta\theta_d \rangle \leq \|\delta\theta_d\|_2 \|\delta\Delta\theta_d\|_2$, (2) from the fact that both $\mu_d$ and $\mu_\infty$ are concentrated in a ball of radius $R$: $\|\delta\theta_d\|_2 = \|\theta_d-\theta_\infty\|_2 \leq \|\theta_d\|_2 + \|\theta_\infty\|_2 \leq 2 R$, and (3) from Jensen's inequality: $\int \|\theta\|_2 \, \mu(d\theta) \leq \sqrt{\int \|\theta\|_2^2 \, \mu(d\theta)}$, for $\mu$ being a probability measure.

    The first term converges to zero by the definition of the sequence of measures $\mu_d$.
    Consider the second term:
    \begin{equation*}
        \int \|\delta\Delta\theta_d\|_2^2 \, \mu^j_{d,\infty} (d\theta_d, \, d\theta_\infty) =
    \end{equation*}
    \begin{equation}
        = \int (\delta\Delta a_d)^2 \, \mu^j_{d,\infty} (d\theta_d, \, d\theta_\infty) +
        \label{eq:deltaDeltatheta}
    \end{equation}
    \begin{equation*}
        + \int \|\delta\Delta\ww_d\|_2^2 \, \mu^j_{d,\infty} (d\theta_d, \, d\theta_\infty). 
    \end{equation*}
    Consider then the first term here:
    \begin{eqnarray*}
        \int (\delta\Delta a_d)^2 \, \mu^j_{d,\infty} (d\theta_d, \, d\theta_\infty) 
        =\\=
        \int \Bigl(
            \EE_{\xx,y} \Bigl(
                \nabla_{f_d}\ell \, \phi(\ww_d^{T} \xx)
            -\\-
            \nabla_{f_\infty}\ell \, \phi(\ww_\infty^{T} \xx)
            \Bigr) 
        \Bigr)^2 \mu^j_{d,\infty} (d\theta_d, \, d\theta_\infty)
        =\\=
        \int (
            \EE_{\xx,y} (
                g(\xx, \theta_d) h(\xx, y, \mu_d) 
                -\\- 
                g(\xx, \theta_\infty) h(\xx, y, \mu_\infty)
            )
        )^2 \, 
        \mu^j_{d,\infty} (d\theta_d, \, d\theta_\infty),
    \end{eqnarray*}
    where we have defined
    \begin{equation*}
        g(\xx, \theta) = 
        g(\xx, (a,\ww)) = 
        \phi(\ww^T \xx),
    \end{equation*}
    \begin{equation*}
        h(\xx, y, \mu) = 
        \left.\frac{\partial \ell(y,z)}{\partial z} \right|_{z=f[\mu; \xx]}.
    \end{equation*}

    W.l.o.g. assume $\phi$ has a Lipschitz constant 1: $\phi(\cdot) \in \Lip(\RR; 1)$.
    From this follows that $g(\xx, \cdot) \in \Lip(\RR^{1+d_0}; \|\xx\|_2)$.
    It is easy to see that since we consider measures supported on $\BB_R$, $f[\cdot,\xx] \in \Lip(\MM(\BB_R^{1+d_0}); 2 R \|\xx\|_2)$ wrt the 2-Wasserstein metric.
    Indeed,
    \begin{eqnarray*}
            |f[\mu_d, \xx] - f[\mu_\infty, \xx]
            =\\=
            \Bigl| \int a_d \phi(\ww_d^T \xx) \, \mu(da_d, \, d\ww_d) 
            -\\- 
            \int a_\infty \phi(\ww_\infty^T \xx) \, \mu(da_\infty, \, d\ww_\infty) \Bigr|
            =\\=
            \Bigl| \int (a_d \phi(\ww_d^T \xx) - a_\infty \phi(\ww_\infty^T \xx)) \, \mu(d\theta_d) \, \mu(d\theta_\infty) \Bigr|
            \leq\\\leq
            \int |a_d \phi(\ww_d^T \xx) - a_\infty \phi(\ww_\infty^T \xx)| \, \mu(d\theta_d) \, \mu(d\theta_\infty)
            \leq\\\leq
            \int (|a_d| |\phi(\ww_d^T \xx) - \phi(\ww_\infty^T \xx)| +\\+ |a_d - a_\infty| |\phi(\ww_\infty^T \xx)|) \, \mu(d\theta_d) \, \mu(d\theta_\infty)
            \leq\\\leq
            R \|\xx\|_2 \int (\|\delta\ww_d\|_2 + |\delta a_d|) \, \mu(d\theta_d) \, \mu(d\theta_\infty)
            \leq\\\leq
            R \|\xx\|_2 \sqrt{\int \|\ww_d - \ww_\infty\|_2^2 \, \mu(d\theta_d) \, \mu(d\theta_\infty)} +\\+
            R \|\xx\|_2 \sqrt{\int |a_d - a_\infty|^2 \, \mu(d\theta_d) \, \mu(d\theta_\infty)}
            \leq\\\leq
            2 R \|\xx\|_2 \WW_2(\mu_d, \mu_\infty),
    \end{eqnarray*}
    where we have used Jensen's inequality: $\int \|\theta\|_2 \, \mu(d\theta) \leq \sqrt{\int \|\theta\|_2^2 \, \mu(d\theta)}$ since $\mu$ is a probability measure.

    W.l.o.g. $\partial\ell/\partial z \in \Lip(\RR; 1)$ $\forall y \in \{0,1\}$. 
    Hence the latter implies $h(\xx, y, \cdot) \in \Lip(\MM(\BB_R^{1+d_0}); 2 R \|\xx\|_2)$.
    
    Taking into account that w.l.o.g. $\partial\ell/\partial z$ and $\phi'$ are bounded by 1, we have:
    \begin{eqnarray*}
        |g(\xx, \theta_d) h(\xx, y, \mu_d) - g(\xx, \theta_\infty) h(\xx, y, \mu_\infty)| 
        \leq\\\leq
        |g(\xx, \theta_d) - g(\xx, \theta_\infty)| 
        +\\+ 
        R \|\xx\|_2 |h(\xx, y, \mu_d) - h(\xx, y, \mu_\infty)| 
        \leq\\\leq
        \|\xx\|_2 \|\theta_d - \theta_\infty\|_2 + 2 R^2 \|\xx\|_2^2 \WW_2(\mu_d, \mu_\infty).
    \end{eqnarray*}
    From this follows:
    \begin{eqnarray*}
        (\EE_{\xx,y} (g(\xx, \theta_d) h(\xx, y, \mu_d) - g(\xx, \theta_\infty) h(\xx, y, \mu_\infty)))^2
        \leq\\\leq
        \EE_{\xx,y} (g(\xx, \theta_d) h(\xx, y, \mu_d) - g(\xx, \theta_\infty) h(\xx, y, \mu_\infty))^2
        \leq\\\leq
        \EE_{\xx,y} \|\xx\|_2^2 \|\theta_d - \theta_\infty\|_2^2 + 4 R^4 \EE_{\xx,y} \|\xx\|_2^4 \WW_2^2(\mu_d, \mu_\infty) +\\+ 4 R^2 \EE_{\xx,y} \|\xx\|_2^3 \|\theta_d - \theta_\infty\|_2 \WW_2(\mu_d, \mu_\infty).
    \end{eqnarray*}

    Hence
    \begin{multline*}
        \lim_{j \to \infty} \int (\EE_{\xx,y} (g(\xx, \theta_d) h(\xx, y, \mu_d) 
        -\\- 
        g(\xx, \theta_\infty) h(\xx, y, \mu_\infty)))^2 \, 
        \mu^j_{d,\infty} (d\theta_d, \, d\theta_\infty)
        \leq\\\leq
        \EE_{\xx,y} \|\xx\|_2^2 \WW_2^2(\mu_d, \mu_\infty) + 4 R^4 \EE_{\xx,y} \|\xx\|_2^4 \WW_2^2(\mu_d, \mu_\infty) +\\+ 4 R^2 \EE_{\xx,y} \|\xx\|_2^3 \WW_2^2(\mu_d, \mu_\infty)
        =\\=
        \EE_{\xx,y} (\|\xx\|_2 + 2 R^2 \|\xx\|_2^2)^2 \WW_2^2(\mu_d, \mu_\infty).
    \end{multline*}

    We can apply the same logic to the second term of (\ref{eq:deltaDeltatheta}) to get the same upper bound:
    \begin{eqnarray*}
        \int (\delta\Delta\ww_d)^2 \, \mu^j_{d,\infty} (d\theta_d, \, d\theta_\infty) 
        =\\=
        \int \Bigl(
            \EE_{\xx,y} \left(
                \nabla_{f_d}\ell \, a_d \phi'(\ww_d^{T} \xx)
            \right. -\\- \left. 
            \nabla_{f_\infty}\ell \, a_\infty \phi'(\ww_\infty^{T} \xx)
            \right) 
        \Bigr)^2 \mu^j_{d,\infty} (d\theta_d, \, d\theta_\infty)
        \leq\\\leq
        \EE_{\xx,y} (\|\xx\|_2 + 2 R^2 \|\xx\|_2^2)^2 \WW_2^2(\mu_d, \mu_\infty).
    \end{eqnarray*}

    Applying this upper bound to equation (\ref{eq:W_2_tau_mu_upper_bound}), we finally get the following:
    \begin{eqnarray*}
        \lim_{d \to \infty} \WW_2^2(\TT\mu_d, \TT\mu_\infty) 
        \leq
        \lim_{d \to \infty} \Bigl(
            \WW_2^2(\mu_d, \mu_\infty) +\\+ 2 \EE_{\xx,y} (\|\xx\|_2 + 2 R^2 \|\xx\|_2^2)^2 \WW_2^2(\mu_d, \mu_\infty) +\\+ 4 R \sqrt{2 \EE_{\xx,y} (\|\xx\|_2 + 2 R^2 \|\xx\|_2^2)^2} \WW_2(\mu_d, \mu_\infty)
        \Bigr) = 0,
    \end{eqnarray*}
    where the last equality is valid, because by assumptions $\xx$ has finite moments up to the fourth one.
    Hence $\TT \mu_d$ converges to $\TT \mu_\infty$ wrt the 2-Wasserstein metric.

    Summing up, we have proven that $\TT$ is continuous wrt the 2-Wasserstein metric.

\end{proof}

\section{The mean-field limit is trivial for the case of more than two hidden layers}
\label{sec:appendix_mf_limit_is_trivial}

Here we re-write the definition of a multi-layer net, as well as the gradient descent step on scaled quantities:
\begin{equation*}
    f(\xx; \aa, V^{1:H}, W) =
    \sum_{r_H=1}^d a_{r_H} \phi(f^H_{r_H}(\xx; V^{1:H}, W)),
\end{equation*}
where
\begin{equation*}
    f^{h+1}_{r_{h+1}}(\xx; V^{1:h+1}, W) =
    \sum_{r_h=1}^d v^{h+1}_{r_{h+1} r_h} \phi(f^h_{r_h}(\xx; V^{1:h}, W)),
\end{equation*}
\begin{equation*}
    f^0_{r_0}(\xx, W) =
    \ww_{r_0}^T \xx.
\end{equation*}
The gradient descent step:
\begin{equation*}
    \Delta\hat a_{r_H}^{(k)} =
    -\hat\eta_a \sigma^{H+1} \EE \nabla_f^{(k)} \ell \; \phi(\hat f^{H,(k)}_{r_H}(\xx)),
\end{equation*}
\begin{equation*}
    \Delta\hat v^{H,(k)}_{r_H r_{H-1}} =
    -\hat\eta_{v^H} \sigma^{H+1} \EE \nabla_f^{(k)} \ell \; \hat a_{r_H}^{(k)} \phi(\hat f^{H-1,(k)}_{r_{H-1}}(\xx)),
\end{equation*}
$$\ldots$$
\begin{multline*}
    \Delta\hat\ww_{r_0}^{(k)} =
    -\hat\eta_w \sigma^{H+1} \EE \nabla_f^{(k)} \ell \; 
    \sum_{r_H=1}^d \hat a^{(k)}_{r_H} \phi'(\hat f^{H,(k)}_{r_H}(\xx)) 
    \times\\\times
    \sum_{r_{H-1}=1}^d \hat v^{H,(k)}_{r_H r_{H-1}} \phi'(\hat f^{H-1,(k)}_{r_{H-1}}(\xx)) 
    \times\ldots\\\ldots\times
    \sum_{r_{1}=1}^d \hat v^{2,(k)}_{r_2 r_{1}} \phi'(\hat f^{1,(k)}_{r_1}(\xx)) 
    \hat v^{1,(k)}_{r_1 r_0} \phi'(\hat\ww_{r_0}^{(k),T} \xx) \xx.
\end{multline*}
\begin{equation}
    \hat a^{(0)}_{r_H} \sim \NN(0, I), \;
    \hat v^{h,(0)}_{r_h r_{h-1}} \sim \NN(0, I), \;
    \hat\ww^{(0)}_{r_0} \sim \NN(0, I),
    \label{eq:gd_dynamics_Hhid_appendix}
\end{equation}
where we have denoted $\hat f^{h,(k)}_{r_h}(\xx) = f^h_{r_h}(\xx; \hat V^{(k),1:h}, \hat W^{(k)})$.

Similarly to the case of $H=0$ (see Section~3), we consider a power-law dependence on $d$ for $\sigma$ and learning rates, as a result introducing $q_\sigma$, $\tilde q_a$, $\tilde q_{v^h}$ and $\tilde q_w$.
In Section~4 we have shown that for the mean-field limit we should have $q_\sigma = -1$, $\tilde q_{a/w} = 1$ and $\tilde q_{v^h} = 2$.

We now show that for $H \geq 2$ the mean-field limit is trivial: $\lim_{d \to \infty} f_d^{(k)}(\xx) = 0$.
Similarly to the case of $H=0$, we introduce weight increments $\delta\hat a_{r_H}^{(k)} = \hat a_{r_H}^{(k)} - \hat a_{r_H}^{(0)}$, $\delta\hat v_{r_h r_{h-1}}^{h,(k)} = \hat v_{r_h r_{h-1}}^{h,(k)} - \hat v_{r_h r_{h-1}}^{h,(0)}$ and $\delta\hat\ww_{r_0}^{(k)} = \hat\ww_{r_0}^{(k)} - \hat\ww_{r_0}^{(0)}$, and assume a power-law dependence on $d$ for them resulting in the introduction of exponents $q_a^{(k)}$, $q_{v^h}^{(k)}$ and $q_w^{(k)}$.

Analogically to a single hidden layer case, we decompose our $f$:
\begin{multline}
    f_d^{(k)}(\xx) =
    f_{d,\emptyset}^{(k)}(\xx) + f_{d,a}^{(k)}(\xx) + \sum_{h=1}^H f_{d,v^h}^{(k)}(\xx) + f_{d,w}^{(k)}(\xx) +\\+ \ldots + f_{d,av^{1:H}w}^{(k)}(\xx),
    \label{eq:f_decomposition_Hhid}
\end{multline}
where the exact definition of each term can be derived from its sub-index: e.g. $f_{d,wa}^{(k)}$ has $\delta\hat a^{(k)}$, $\delta\hat\ww^{(k)}$ and $\hat v^{h,(0)}$ $\forall h \in [H]$ terms.

Introducing an exponent $q$ for each term, we get:
\begin{equation}
    q_f^{(k)} = \max(q_{f,\emptyset}^{(k)}, q_{f,a}^{(k)}, \ldots, q_{f,av^{1:H}w}^{(k)}), \quad
    q_f^{(0)} = 2 q_\sigma + 1.
\end{equation}
We write all of the terms of the decomposition for $f$ in a unified way.
Let $\Theta_h$ be a subset of $\{a, v^{1:H}, w\}$ of size $h$.
Then:
\begin{equation}
    q_{f,\Theta_h}^{(k)} =
    H (\kappa_{\Theta_h}^{(k)} + q_\sigma) + \sum_{\theta \in \Theta_h} q_{\theta}^{(k)}, 
    \label{eq:q_fk_theta}
\end{equation}
where $\varkappa_{\Theta_h}^{(k)} \in [1/2,1]$ comes from the same logic as in the single hidden layer case.
Since $q_\sigma = -1$, if we show that all $q_\theta^{(k)} < 0$ $\forall k \geq 1$, then we conclude that all components of decomposition~(\ref{eq:f_decomposition_Hhid}) vanish.

Let us look on the gradient descent dynamics (\ref{eq:gd_dynamics_Hhid_appendix}).
It implies the following equalities for $k = 0$:
\begin{equation}
    q_{a/w}^{(1)} = 
    \tilde q_{a/w} + (H+1) q_\sigma + \frac{H}{2} =
    -\frac{H}{2},
    \label{eq:q1_Hhid_ineq}
\end{equation}
\begin{equation*}
    q_{v^h}^{(1)} = 
    \tilde q_{v^h} + (H+1) q_\sigma + \frac{H-1}{2} =
    -\frac{H-1}{2},
\end{equation*}
which come from the fact that all $\hat a^{(0)}$, $\hat v^{h,(0)}$ and $\hat\ww^{(0)}$ are independent and $\propto 1$. Indeed, gradient updates for $\delta\hat a$ and $\delta\hat\ww$ have $H$ sums each, and each sum scales as $d^{1/2}$ (this where the term $H/2$ comes from); at the same time gradient updates for $\delta\hat v^h$ have $H-1$ sums each.

Due to the symmetry of the gradient step dynamics, $q_{v^1}^{(1)} = \ldots = q_{v^H}^{(1)}$ imply $q_{v^1}^{(k)} = \ldots = q_{v^H}^{(k)}$ $\forall k \geq 1$.
We shall denote it with $q_v^{(k)}$ then.

Suppose $H \geq 2$.
We prove that $q_{a/w}^{(k)} \leq q_{a/w}^{(1)} = -H/2$ and $q_{v}^{(k)} \leq q_{v}^{(1)} = (1-H)/2$ $\forall k \geq 1$ by induction.
The induction base $k=1$ is trivial. 
For the sake of illustration, we first consider the induction step for $q_w$:
\begin{multline}
    q_{w}^{(k+1)} 
    \leq
    \max\Biggl(q_{w}^{(k)}, \tilde q_{w} + (H+1) q_\sigma +\\+ \max\Biggl(\frac{H}{2}, \frac{H+1}{2} + q_{a}^{(k)}, \frac{H+1}{2} + q_{v}^{(k)}, \\ H + q_{a}^{(k)} + q_{v}^{(k)}, H + 2 q_{v}^{(k)}\Biggr)\Biggr) 
    \leq\\\leq
    \max\Biggl(-\frac{H}{2}, -\frac{H}{2} + \max\Biggl(0, \frac{1}{2} + q_{a}^{(k)}, \frac{1}{2} + q_{v}^{(k)},\\ \frac{H}{2} + q_{a}^{(k)} + q_{v}^{(k)}, \frac{H}{2} + 2 q_{v}^{(k)}\Biggr)\Biggr) 
    \leq\\\leq
    \max\Biggl(-\frac{H}{2}, -\frac{H}{2} + \max\Biggl(0, \frac{1-H}{2}, \frac{2-H}{2}, \\ \frac{1-H}{2},  \frac{2-H}{2}\Biggr)\Biggr) 
    = -\frac{H}{2}.
    \label{eq:q_w_k1}
\end{multline}
All inequalities except the first come from the induction hypothesis.
We now demonstrate where the first inequality comes from. 
Recall that $\|\delta\hat\ww^{(k+1)}\| \propto d^{q_w^{(k+1)}}$ and
\begin{multline}
    \delta\hat\ww_{r_0}^{(k+1)} =
    \delta\hat\ww_{r_0}^{(k)} 
    -\\-
    \hat\eta_w \sigma^{H+1} \EE \nabla_f^{(k)} \ell \; 
    \sum_{r_H=1}^d (\hat a^{(0)}_{r_H} + \delta\hat a^{(k)}_{r_H}) \phi'(\hat f^{H,(k)}_{r_H}(\xx)) 
    \times\\\times
    \sum_{r_{H-1}=1}^d (\hat v^{H,(0)}_{r_H r_{H-1}} + \delta\hat v^{H,(k)}_{r_H r_{H-1}}) \phi'(\hat f^{{H-1},(k)}_{r_{H-1}}(\xx)) 
    \times\ldots\\\ldots\times
    \sum_{r_{1}=1}^d (\hat v^{2,(0)}_{r_2 r_{1}} + \delta\hat v^{2,(k)}_{r_2 r_{1}}) \phi'(\hat f^{1,(k)}_{r_1}(\xx)) 
    \times\\\times
    (\hat v^{1,(0)}_{r_1 r_0} + \delta\hat v^{1,(k)}_{r_1 r_0}) \phi'((\hat\ww_{r_0}^{(0)} + \delta\hat\ww_{r_0}^{(k)})^T \xx) \xx.
    \label{eq:gradient_step_w}
\end{multline}
Here we have a product of $H$ sums, by expanding which we obtain a sum of $2^{H+1}$ products of sums in total; for example, for $H=2$ we have:
\begin{multline*}
    \sum_{r_2=1}^d (\hat a^{(0)}_{r_2} + \delta\hat a^{(k)}_{r_2}) \phi'(\hat f^{2,(k)}_{r_2}(\xx)) 
    \times\\\times
    \sum_{r_{1}=1}^d (\hat v^{2,(0)}_{r_2 r_{1}} + \delta\hat v^{2,(k)}_{r_2 r_{1}}) \phi'(\hat f^{1,(k)}_{r_1}(\xx)) 
    \times\\\times
    (\hat v^{1,(0)}_{r_1 r_0} + \delta\hat v^{1,(k)}_{r_1 r_0}) \phi'((\hat\ww_{r_0}^{(0)} + \delta\hat\ww_{r_0}^{(k)})^T \xx) \xx
    =\\=
    \sum_{r_2=1}^d \hat a^{(0)}_{r_2} \phi'(\ldots) 
    \sum_{r_{1}=1}^d \hat v^{2,(0)}_{r_2 r_{1}} \phi'(\ldots) 
    \hat v^{1,(0)}_{r_1 r_0} \phi'(\ldots) \xx
    +\\+
    \sum_{r_2=1}^d \delta\hat a^{(k)}_{r_2} \phi'(\ldots) 
    \sum_{r_{1}=1}^d \hat v^{2,(0)}_{r_2 r_{1}} \phi'(\ldots) 
    \hat v^{1,(0)}_{r_1 r_0} \phi'(\ldots) \xx
    +\\+
    \sum_{r_2=1}^d \hat a^{(0)}_{r_2} \phi'(\ldots) 
    \sum_{r_{1}=1}^d \delta\hat v^{2,(k)}_{r_2 r_{1}} \phi'(\ldots) 
    \hat v^{1,(0)}_{r_1 r_0} \phi'(\ldots) \xx
    +\\+
    \sum_{r_2=1}^d \hat a^{(0)}_{r_2} \phi'(\ldots) 
    \sum_{r_{1}=1}^d \hat v^{2,(0)}_{r_2 r_{1}} \phi'(\ldots) 
    \delta\hat v^{1,(k)}_{r_1 r_0} \phi'(\ldots) \xx
    +\ldots\\\ldots+
    \sum_{r_2=1}^d \delta\hat a^{(k)}_{r_2} \phi'(\ldots) 
    \sum_{r_{1}=1}^d \delta\hat v^{2,(k)}_{r_2 r_{1}} \phi'(\ldots) 
    \delta\hat v^{1,(k)}_{r_1 r_0} \phi'(\ldots) \xx
    =\\=
    \Sigma_{d,\emptyset}^{(k)} + \Sigma_{d,a}^{(k)} + \Sigma_{d,v^1}^{(k)} + \Sigma_{d,v^2}^{(k)} 
    +\\+ 
    \Sigma_{d, v^1 v^2}^{(k)} + \Sigma_{d,v^2 a}^{(k)} + \Sigma_{d,a v^1}^{(k)} + \Sigma_{d, a v^1 v^2}^{(k)},
\end{multline*}
where the notation we have introduced is intuitive: for example, $\Sigma_{d,a v^1}^{(k)} = \sum_{r_2=1}^d \delta\hat a^{(0)}_{r_2} \phi'(\ldots) \sum_{r_{1}=1}^d \hat v^{2,(0)}_{r_2 r_{1}} \phi'(\ldots) \delta\hat v^{1,(k)}_{r_1 r_0} \phi'(\ldots) \xx$.

If we assume power-law dependencies for all $\Sigma$-terms, i.e. $\Sigma_{d,\emptyset}^{(k)} \propto d^{q_{\Sigma,\emptyset}^{(k)}}$, $\Sigma_{d,a}^{(k)} \propto d^{q_{\Sigma,a}^{(k)}}$ and so on, using heuristic rules mentioned in Section~3, from (\ref{eq:gradient_step_w}) we get the following:
\begin{multline*}
    q_{w}^{(k+1)} 
    =
    \max(q_{w}^{(k)}, \tilde q_{w} + (H+1) q_\sigma +\\+ \max(q_{\Sigma,\emptyset}^{(k)}, q_{\Sigma,a}^{(k)}, q_{\Sigma,v^1}^{(k)}, q_{\Sigma,v^2}^{(k)}, \ldots, q_{\Sigma,a v^1 v^2}^{(k)})).
\end{multline*}

First consider $\Sigma_{d,a v^1 v^2}^{(k)}$.
This term is a product of two sums with $d$ terms each.
Since each sum cannot grow faster than $d$, we get the following upper bound:
\begin{equation*}
    q_{\Sigma,a v^1 v^2}^{(k)}
    \leq
    q_a^{(k)} + 2 q_v^{(k)} + 2.
\end{equation*}
Similar upper bounds hold for all other $\Sigma$-terms; in particular, we have:
\begin{equation*}
    q_{\Sigma,v^1 v^2}^{(k)}
    \leq
    2 q_v^{(k)} + 2,
    \;
    \max(q_{\Sigma,v^2 a}^{(k)}, q_{\Sigma,a v^1}^{(k)})
    \leq
    q_a^{(k)} + q_v^{(k)} + 2.
\end{equation*}
For $\Sigma_{d,\emptyset}^{(k)}$ we compute the corresponding exponent exactly:
$
    q_{\Sigma,\emptyset}^{(k)} = 1.
$
In this case both sums are the sums of asymptotically independent terms with zero mean.
Indeed, we have:
\begin{eqnarray*}
    \Sigma_{d,\emptyset}^{(k)}
    =
    \sum_{r_2=1}^d \hat a^{(0)}_{r_2} \phi'(\hat f_{r_2}^{2,(k)}(\xx)) 
    \times\\\times
    \sum_{r_{1}=1}^d \hat v^{2,(0)}_{r_2 r_{1}} \phi'(\hat f_{r_1}^{1,(k)}(\xx)) 
    \hat v^{1,(0)}_{r_1 r_0} \phi'((\hat\ww_{r_0}^{(0)} + \delta\hat\ww_{r_0}^{(k)})^T \xx) \xx
    \sim\\\sim
    \sum_{r_2=1}^d \hat a^{(0)}_{r_2} \phi'(\hat f_{r_2}^{2,(0)}(\xx)) 
    \times\\\times
    \sum_{r_{1}=1}^d \hat v^{2,(0)}_{r_2 r_{1}} \phi'(\hat f_{r_1}^{1,(0)}(\xx)) 
    \hat v^{1,(0)}_{r_1 r_0} \phi'(\hat\ww_{r_0}^{(0),T} \xx) \xx,
\end{eqnarray*}
where the asymptotic equivalence takes place, because by the induction hypothesis $q_{a/w}^{(k)} \leq -H/2 < 0$ and $q_v^{(k)} \leq (1-H)/2 < 0$.

Finally, let us consider "linear" terms, i.e. $\Sigma_{d,a}^{(k)}$, $\Sigma_{d,v^1}^{(k)}$, $\Sigma_{d,v^2}^{(k)}$.
We consider $\Sigma_{d,a}^{(k)}$ for simplicity; two other terms can be analysed in a similar manner.
Here we have a similar asymptotic relation as we had for $\Sigma_{d,\emptyset}^{(k)}$:
\begin{eqnarray*}
    \Sigma_{d,a}^{(k)}
    =
    \sum_{r_2=1}^d \delta\hat a^{(k)}_{r_2} \phi'(\hat f_{r_2}^{2,(k)}(\xx)) 
    \times\\\times
    \sum_{r_{1}=1}^d \hat v^{2,(0)}_{r_2 r_{1}} \phi'(\hat f_{r_1}^{1,(k)}(\xx)) 
    \hat v^{1,(0)}_{r_1 r_0} \phi'((\hat\ww_{r_0}^{(0)} + \delta\hat\ww_{r_0}^{(k)})^T \xx) \xx
    \sim\\\sim
    \sum_{r_2=1}^d \delta\hat a^{(k)}_{r_2} \phi'(\hat f_{r_2}^{2,(0)}(\xx)) 
    \times\\\times
    \sum_{r_{1}=1}^d \hat v^{2,(0)}_{r_2 r_{1}} \phi'(\hat f_{r_1}^{1,(0)}(\xx)) 
    \hat v^{1,(0)}_{r_1 r_0} \phi'(\hat\ww_{r_0}^{(0),T} \xx) \xx.
\end{eqnarray*}
Let us now recall the gradient step for $\delta\hat a^{(k)}$:
\begin{eqnarray*}
    \delta\hat a_{r_2}^{(k)} =
    \delta\hat a_{r_2}^{(k-1)}
    -
    \hat\eta_a \sigma^{3} \EE \nabla_f^{(k-1)} \ell \;
    \times\\\times 
    \phi\left(\sum_{r_1=1}^d (\hat v_{r_2 r_1}^{2, (0)} + \delta\hat v_{r_2 r_1}^{2, (k-1)}) \phi(\hat f^{1,(k-1)}_{r_1}(\xx))\right).
\end{eqnarray*}
Since by the induction hypothesis $q_v^{(k-1)} \leq (1-H)/2 < 0$, $\delta\hat a_{r_2}^{(k)}$ depends on $\hat v_{r_2 r_1}^{2, (0)}$, even as $d \to \infty$.
This means that the sum over $r_2$ in the definition of $\Sigma_{d,a}^{(k)}$ above grows as $d$, while the sum over $r_1$ still grows as $d^{1/2}$, as was the case for $\Sigma_{d,\emptyset}^{(k)}$.
Hence
\begin{equation*}
    q_{\Sigma, a/v^1/v^2}^{(k)} = 
    q_{a/v^1/v^2}^{(k)} + 3/2.
\end{equation*}

Finally, for $H=2$ we get the following:
\begin{multline*}
    q_{w}^{(k+1)} 
    =
    \max(q_{w}^{(k)}, \tilde q_{w} + (H+1) q_\sigma 
    +\\+ 
    \max(q_{\Sigma,\emptyset}^{(k)}, q_{\Sigma,a}^{(k)}, q_{\Sigma,v^1}^{(k)}, q_{\Sigma,v^2}^{(k)}, \ldots, q_{\Sigma,a v^1 v^2}^{(k)}))
    \leq\\\leq
    \max(q_{w}^{(k)}, \tilde q_{w} + (H+1) q_\sigma + \max(1, q_a^{(k)} + 3/2, q_{v}^{(k)} + 3/2,
    \\ 
    q_a^{(k)} + q_{v}^{(k)} + 2, 2 q_{v}^{(k)} + 2, q_a^{(k)} + 2 q_{v}^{(k)} + 2))
    =\\=
    \max(q_{w}^{(k)}, \tilde q_{w} + (H+1) q_\sigma +\\+ \max(1, q_a^{(k)} + 3/2, q_{v}^{(k)} + 3/2, q_a^{(k)} + q_{v}^{(k)} + 2, 2 q_{v}^{(k)} + 2)),
\end{multline*}
where the last equality comes from the fact that $q_{a/v/w}^{(k)} < 0$ by the induction hypothesis.
Directly extending this technique to the case of $H \geq 2$ results in the first inequality of (\ref{eq:q_w_k1}).

Applying the similar technique to $q_a$ and $q_v$ we get the following:
\begin{multline*}
    q_{a}^{(k+1)} 
    \leq
    \max\Biggl(q_{a}^{(k)}, \tilde q_{a} + (H+1) q_\sigma +\\+ \max\Biggl(\frac{H}{2}, \frac{H+1}{2} + q_{w}^{(k)}, \frac{H+1}{2} + q_{v}^{(k)}, \\ H + q_{w}^{(k)} + q_{v}^{(k)}, H + 2 q_{v}^{(k)}\Biggr)\Biggr) 
    \leq\\\leq
    \max\Biggl(-\frac{H}{2}, -\frac{H}{2} + \max\Biggl(0, \frac{1}{2} + q_{w}^{(k)}, \frac{1}{2} + q_{v}^{(k)},\\ \frac{H}{2} + q_{w}^{(k)} + q_{v}^{(k)}, \frac{H}{2} + 2 q_{v}^{(k)}\Biggr)\Biggr) 
    \leq\\\leq
    \max\Biggl(-\frac{H}{2}, -\frac{H}{2} + \max\Biggl(0, \frac{1-H}{2}, \frac{2-H}{2}, \\ \frac{1-H}{2},  \frac{2-H}{2}\Biggr)\Biggr) 
    = -\frac{H}{2},
\end{multline*}
\begin{multline*}
    q_{v}^{(k+1)} 
    \leq
    \max\Biggl(q_{v}^{(k)}, \tilde q_{v} + (H+1) q_\sigma +\\+ \max\Biggl(\frac{H-1}{2}, \frac{H}{2} + q_{a}^{(k)}, \frac{H}{2} + q_{w}^{(k)}, \frac{H}{2} + q_{v}^{(k)}, \\ H-1 + q_{a}^{(k)} + q_{w}^{(k)}, H-1 + q_{w}^{(k)} + q_{v}^{(k)}, H-1 + q_{a}^{(k)} + q_{v}^{(k)}\Biggr)\Biggr) 
    \leq\\\leq
    \max\Biggl(\frac{1-H}{2}, \frac{1-H}{2} + \max\Biggl(0, \frac{1}{2} + q_{a}^{(k)}, \frac{1}{2} + q_{w}^{(k)}, \\ \frac{1}{2} + q_{v}^{(k)}, \frac{H-1}{2} + q_{a}^{(k)} + q_{w}^{(k)}, \frac{H-1}{2} + q_{w}^{(k)} + q_{v}^{(k)}, \\ \frac{H-1}{2} + q_{a}^{(k)} + q_{v}^{(k)}\Biggr)\Biggr) 
    \leq\\\leq
    \max\Biggl(\frac{1-H}{2}, \frac{1-H}{2} + \max\Biggl(0, \frac{1-H}{2}, \frac{1-H}{2}, \\ \frac{2-H}{2}, \frac{-1-H}{2}, -\frac{H}{2}, -\frac{H}{2}\Biggr)\Biggr) 
    = \frac{1-H}{2},
\end{multline*}
for all $h \in [H]$.
The difference between $q_{a/w}$ and $q_v$ comes from the fact that the gradient step for $\delta \hat v^h$ has $H-1$ sums instead of $H$.

Summing up, we have proven by induction that $\forall k \geq 1$ $q_{a/w}^{(k)} \leq q_{a/w}^{(1)} = -H/2 < 0$ and $q_{v}^{(k)} \leq q_{v}^{(1)} = (1-H)/2 < 0$.
Hence due to (\ref{eq:q_fk_theta}), $q_{f,\Theta_h}^{(k)} < 0$, hence all components of decomposition (\ref{eq:f_decomposition_Hhid}) vanish and $\lim_{d \to \infty} f_d^{(k)} = 0$.

\section{Comparing scalings for small learning rates}
\label{sec:appendix_small_lr_comparison}

As we have noted in Section~3, the MF limit provides the most accurate approximation for a finite-width reference network.
However as we demonstrate here the NTK limit becomes the most accurate approximation for a finite-width reference network if learning rates are sufficiently small and the number of training steps is fixed.

Figure \ref{fig:mf_ntk_1hid_cifar2_sgd_test_loss_and_var_f_small_lr} shows results for two different setups: training a one hidden layer net with gradient descent for 50 epochs with reference learning rates $\eta_a^* = \eta_w^* = 0.02$ (the same setup as in Figure \ref{fig:1hid_cifar2_sgd_da_dw_q} of Section 3 of the main text) and the same setup but with $\eta_a^* = \eta_w^* = 0.0002$.
As one can see, MF and intermediate limits do not preserve the variance of the CE loss but the NTK limit does.

In Section~3 we have argued that the MF limit provides a better approximation for a finite-width reference net, because all terms of decomposition (\ref{eq:f_decomposition_1hid_appendix}) are preserved, however, as we have previously observed in SM~\ref{sec:appendix_sanity_checks}, the term $f_{d,\emptyset}^{(k)}$ is not strictly preserved but approaches a non-zero constant for large $d$.
As we observe in the right plot of the bottom row, the width $2^{16} = 65536$ is not yet enough for $f_{d,\emptyset}^{(k)}$ to reach its asymptotics for the MF limit if learning rates are small: see blue solid curve.
Nevertheless, for large learning rates (right plot of the top row) this term does reach its asymptotics.

However one of the decomposition term vanishes for the NTK limit but for the MF limit it does not: $f_{d,aw}^{(k)}$.
Let us rewrite the definition of this term here:
\begin{equation*}
    f_{d,aw}^{(k)}(\xx) =
    \sigma \sum_{r=1}^d \delta\hat a_r^{(k)} \phi'(\ldots) \delta\hat\ww_r^{(k),T} \xx.
\end{equation*}
This term depends quadratically on weight increments and each weight increment is proportional to a corresponding learning rate.
Hence this term grows quadratically with learning rates.
By the same logic, terms $f_{d,a}^{(k)}$ and $f_{d,w}^{(k)}$ grow linearly with learning rates and $f_{d,\emptyset}^{(k)}$ has no polynomial dependence on learning rates.
This reasoning implies that the term $f_{d,aw}^{(k)}$ vanishes faster than others as learning rates go to zero.
Hence the effect of non-preserving this term becomes negligible if learning rates are small.
Because of this, the advantage of the MF limit over the NTK limit disappears for sufficiently small learning rates.
This effect is clearly shown in the right column of Figure \ref{fig:mf_ntk_1hid_cifar2_sgd_test_loss_and_var_f_small_lr}.
For large learning rates (top row) the term $f_{d,aw}^{(k)}$ is the second-largest term in decomposition (\ref{eq:f_decomposition_1hid_appendix}) of the reference network: see a dash-dot curve, however it becomes negligible for one hundred times smaller learning rates (bottom row).

\begin{figure*}[t]
    \centering
    \includegraphics[width=0.32\textwidth]{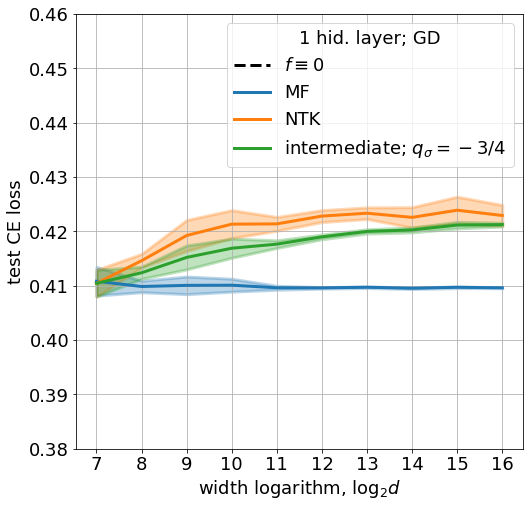}
    \includegraphics[width=0.32\textwidth]{images/cifar2_1hid_lr=002_lrelu_sgd_test_loss_evolution.png}
    \includegraphics[width=0.32\textwidth]{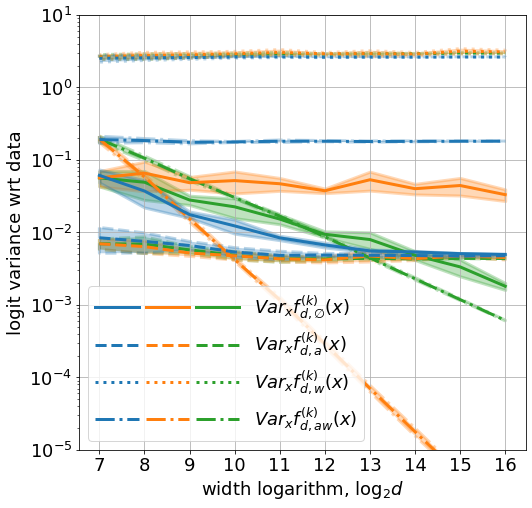}
    \\
    \includegraphics[width=0.32\textwidth]{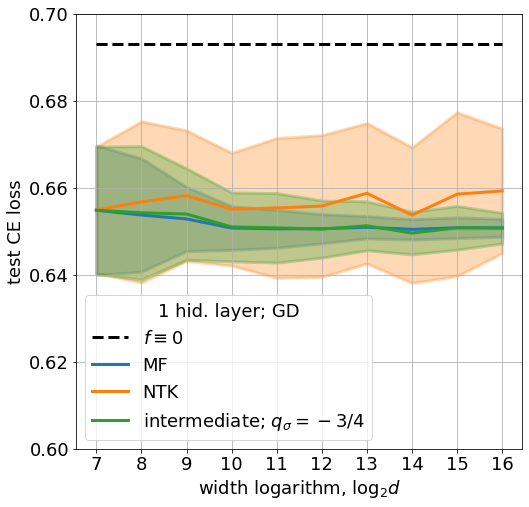}
    \includegraphics[width=0.32\textwidth]{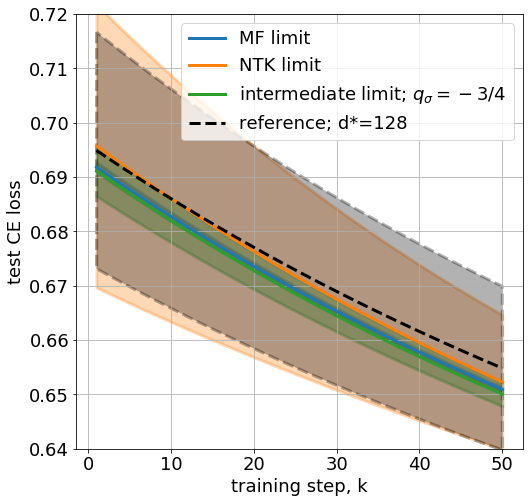}
    \includegraphics[width=0.32\textwidth]{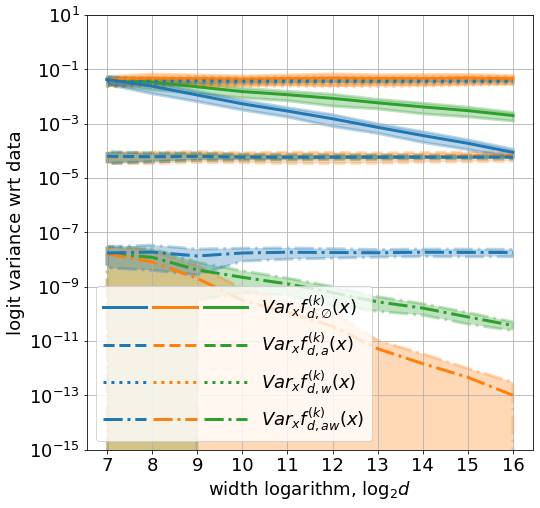}
    \caption{\textbf{For small learning rates, the NTK limit approximates the reference finite-width network better than the MF limit.} \textit{Top row:} scaling a reference network trained with gradient descent with (unscaled) learning rates $\eta_a^* = \eta_w^* = 0.02$. \textit{Bottom row:} same with unscaled learning rates $\eta_a^* = \eta_w^* = 0.0002$. \textit{Left:} a final test cross-entropy (CE) loss as a function of width $d$. \textit{Center:} test CE loss as a function of training step $k$ for a reference net and its limits. As one can see, MF and intermediate limits preserve mean CE loss but not its variance with respect to the initialization. In contrast, the NTK limit does preserve the variance. \textit{Right:} variance with respect to the data distribution for terms of model decomposition (\ref{eq:f_decomposition_1hid_appendix}) as a function of width $d$. When learning rates are small, $f_{d,\emptyset}^{(k)}$, which contributes to the variance, becomes the largest term in decomposition (\ref{eq:f_decomposition_1hid_appendix}) and $f_{d,aw}^{(k)}$, which vanishes in NTK and intermediate limits, becomes the smallest. As we have noticed in Figure \ref{fig:1hid_cifar2_sgd_var_f_decomp} for the MF limit $f_{d,\emptyset}^{(k)}$ is not exactly constant but decays approaching a constant for large $d$. This is the reason for the MF limit not to preserve the variance of CE loss. \textit{Setup:} We train a 1-hidden layer net on a subset of CIFAR2 (a dataset of first two classes of CIFAR10) of size 1000 with gradient descent. We take a reference net of width $d^* = 2^7 = 128$ and scale its hyperparameters according to MF (blue curves), NTK (orange curves) and intermediate scalings with $q_\sigma=-3/4$ (green curves, see text). See SM~\ref{sec:appendix_experiments} for further details.}
    \label{fig:mf_ntk_1hid_cifar2_sgd_test_loss_and_var_f_small_lr}
\end{figure*}

\section{Experiments for other setups}
\label{sec:appendix_other_setups}

As was noted in Section~\ref{sec:appendix_experiments}, in the present work we typically train a network using a full-batch gradient descent (or RMSProp) for 50 epochs (or equivalently, training steps) on a subset of CIFAR2 of size 1000.
The reason for this is that our theory is developed for binary classification, it assumes exact gradient computations, and because training up to convergence is not necessary for our framework.

In this section we experiment with modifications of our usual setup: see Figure \ref{fig:cifar2_1hid_sgd_test_loss_and_var_f} for the case of a one hidden layer net trained with the (stochastic) gradient descent.
The top row represents the usual case of the full-batch gradient descent training for 50 epochs with unscaled reference learning rates $\eta_a^* = \eta_w^* = 0.02$ applied to a subset of CIFAR2.
For the next row we set the batch size to 100, while keeping the number of gradient updates.
As we observe, applying a stochastic gradient descent instead of the full-batch one does not introduce any qualitative changes.
For the third row we take a full CIFAR2 (with training size being 10000 instead of 1000), while keeping the batch size to be 1000.
It is hard to spot any qualitative changes in this setup as well.
For the bottom row we increase the number of epochs (training steps) by the factor of 10, while keeping the rest of the options.
In this case all plots change, which is expected since 50 epochs of the original setup is not enough for convergence of training procedure.
As we observe in the center column, in this case some of the terms of decomposition (\ref{eq:f_decomposition_1hid_appendix}) do not obey power-laws but converge to power-laws for large $d$.

We also consider a multi-class classification instead of a binary one: see Figure \ref{fig:mnist_1hid_sgd_test_loss_and_var_f}.
The top row corresponds to the usual scenario of a binary classification on a subset of CIFAR2 of size 1000; it is given for the reference.
The middle row corresponds to the same scenario but on a subset of MNIST of size 1000; MNIST has 10 classes instead of two.
Comparing these two scenarios does not reveal any qualitative changes.

The bottom row corresponds to the most realistic scenario among the ones we have considered.
Here we train a one hidden layer network on a full MNIST dataset for 6000 gradient steps using a mini-batch gradient descent with batches of size 100.
With this number of epochs the optimization process nearly converges.
As we see, for this scenario the maximum width $d = 2^{16} = 65536$ we were able to afford was not enough to reach the asymptotic regime fully (center column).
This is the reason for discrepancies between numerical estimates of exponents of decomposition (\ref{eq:f_decomposition_1hid_appendix}) terms and their theoretical values (right column).

\begin{figure*}[t]
    \centering
    \includegraphics[width=0.29\textwidth]{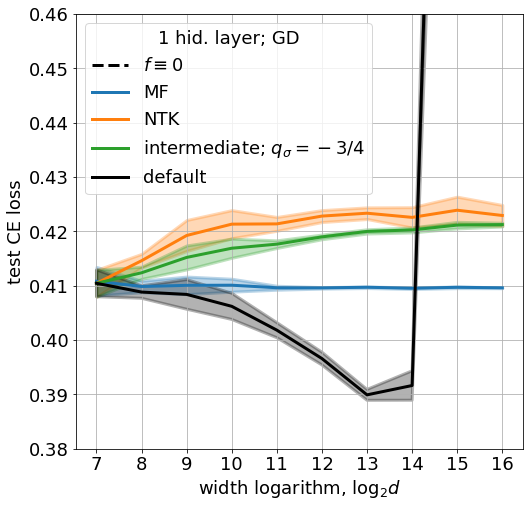}
    \includegraphics[width=0.29\textwidth]{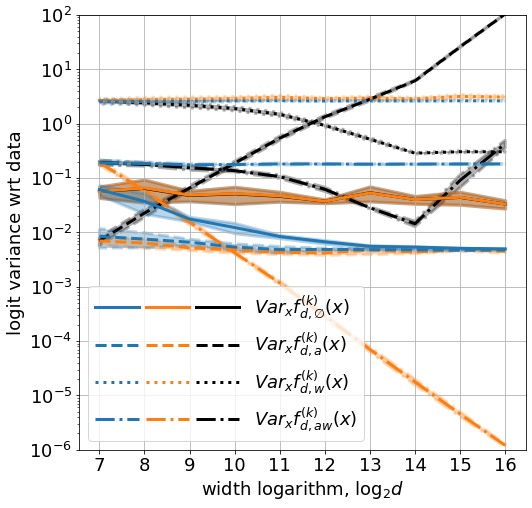}
    \includegraphics[width=0.29\textwidth]{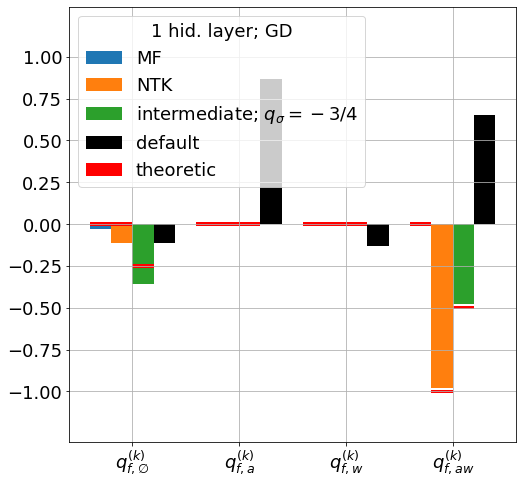}
    \\
    \includegraphics[width=0.29\textwidth]{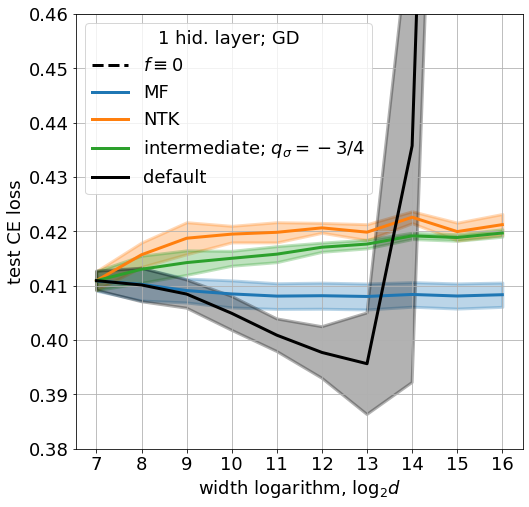}
    \includegraphics[width=0.29\textwidth]{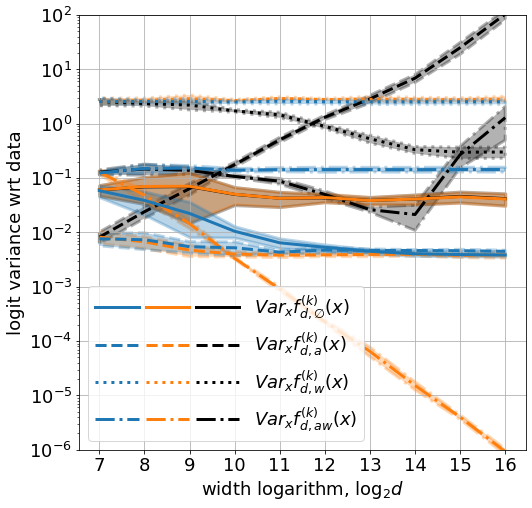}
    \includegraphics[width=0.29\textwidth]{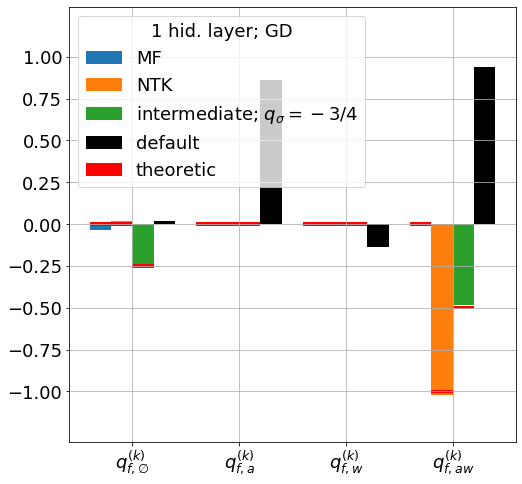}
    \\
    \includegraphics[width=0.29\textwidth]{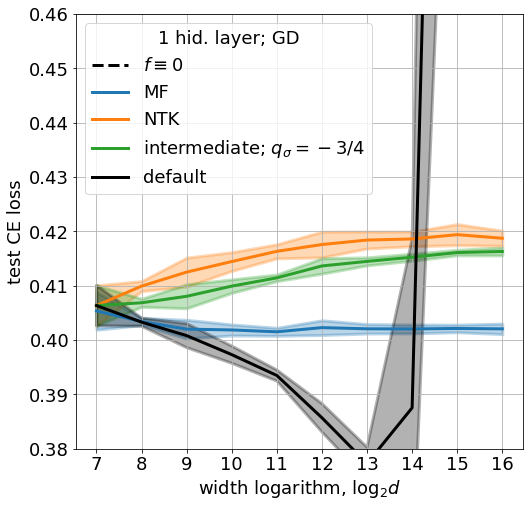}
    \includegraphics[width=0.29\textwidth]{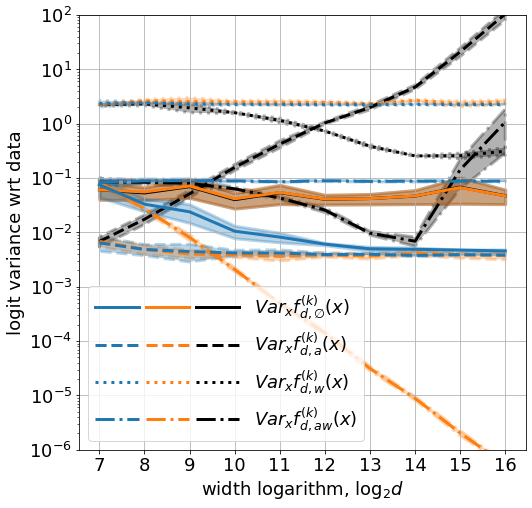}
    \includegraphics[width=0.29\textwidth]{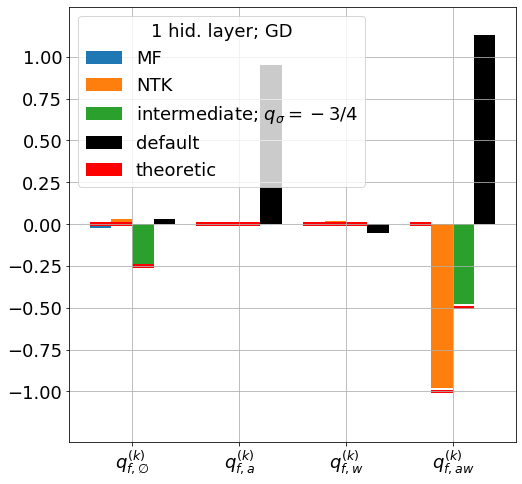}
    \\
    \includegraphics[width=0.29\textwidth]{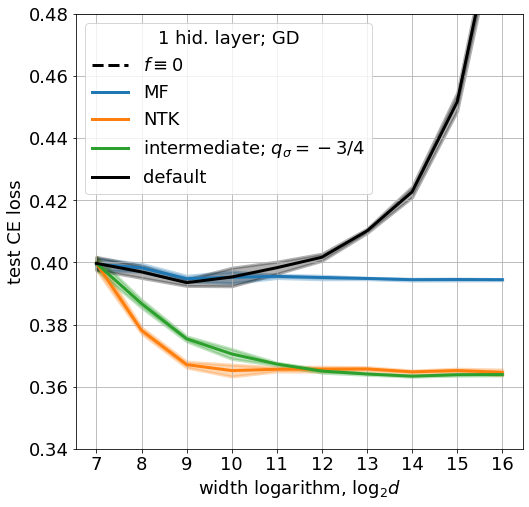}
    \includegraphics[width=0.29\textwidth]{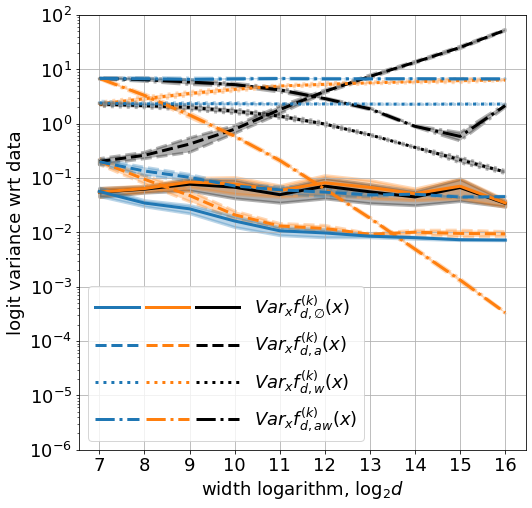}
    \includegraphics[width=0.29\textwidth]{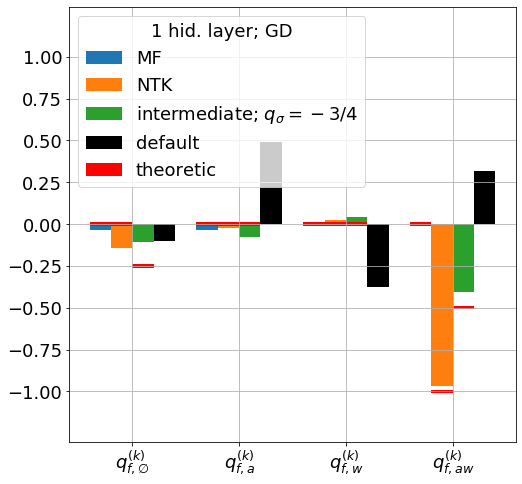}
    \caption{\textbf{Applying a mini-batch instead of a full batch gradient descent does not introduce any qualitative changes. The same holds for training on a larger dataset.} \textit{Top row:} scaling a reference network trained with a full-batch GD with (unscaled) learning rates $\eta_a^* = \eta_w^* = 0.02$ for 50 gradient steps on a subset of CIFAR2 (a dataset of first two classes of CIFAR10) of size 1000. \textit{Second row:} same with a mini-batch GD with batches of size 100. \textit{Third row:} same as the top row but on a full CIFAR2 (10000 training samples) with the mini-batch GD with batches of size 1000. \textit{Bottom row:} same as the top row but with 500 gradient steps. \textit{Left:} a final test cross-entropy (CE) loss as a function of width $d$. \textit{Center:} variance with respect to the data distribution for terms of model decomposition (\ref{eq:f_decomposition_1hid_appendix}) as a function of width $d$. \textit{Right:} numerical estimates for exponents of decomposition (\ref{eq:f_decomposition_1hid_appendix}) terms, as well as their theoretical values (denoted by red ticks). See SM~\ref{sec:appendix_experiments} for further details.}
    \label{fig:cifar2_1hid_sgd_test_loss_and_var_f}
\end{figure*}

\begin{figure*}[t]
    \centering
    \includegraphics[width=0.29\textwidth]{images/cifar2_1hid_lr=002_lrelu_sgd_batch_size=1000_test_loss.png}
    \includegraphics[width=0.29\textwidth]{images/cifar2_1hid_lr=002_lrelu_sgd_batch_size=1000_var_f_decomp_wo_intermediate_d_ref=128.png}
    \includegraphics[width=0.29\textwidth]{images/cifar2_1hid_lr=002_lrelu_sgd_batch_size=1000_q_f_wo_intermediate.png}
    \\
    \includegraphics[width=0.29\textwidth]{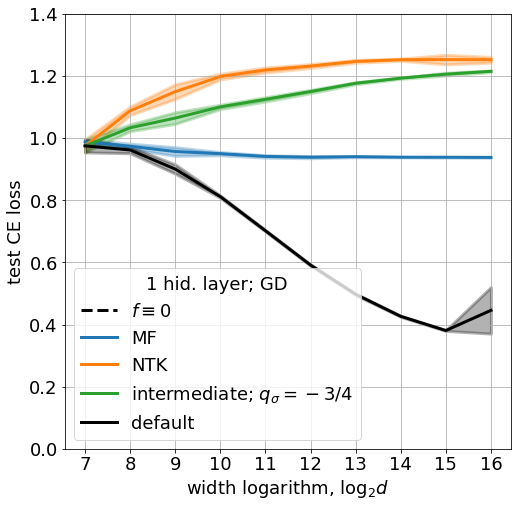}
    \includegraphics[width=0.29\textwidth]{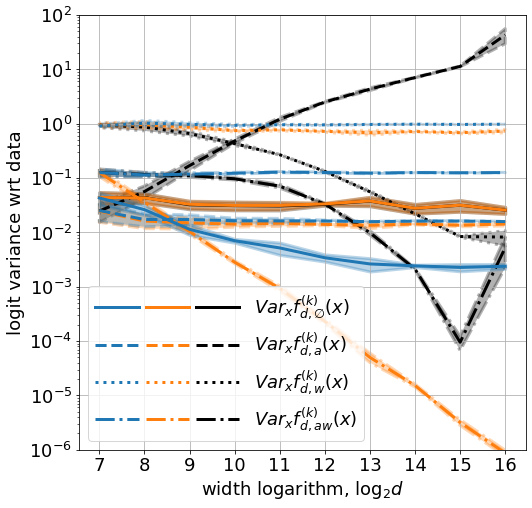}
    \includegraphics[width=0.29\textwidth]{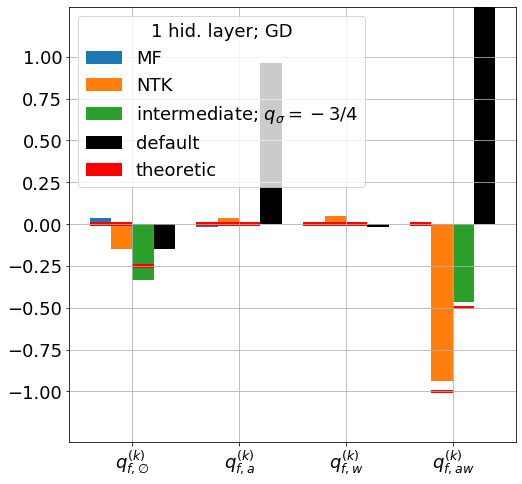}
    \\
    \includegraphics[width=0.29\textwidth]{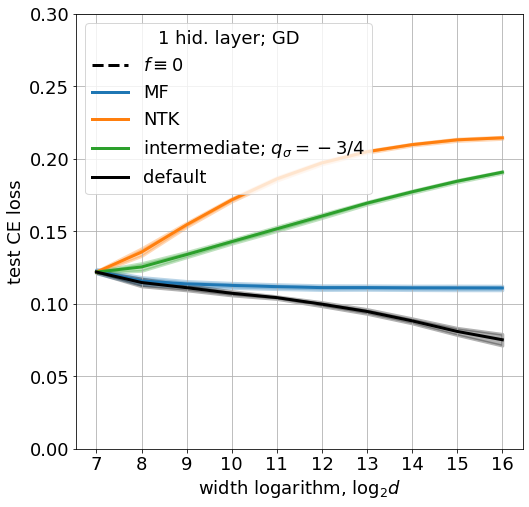}
    \includegraphics[width=0.29\textwidth]{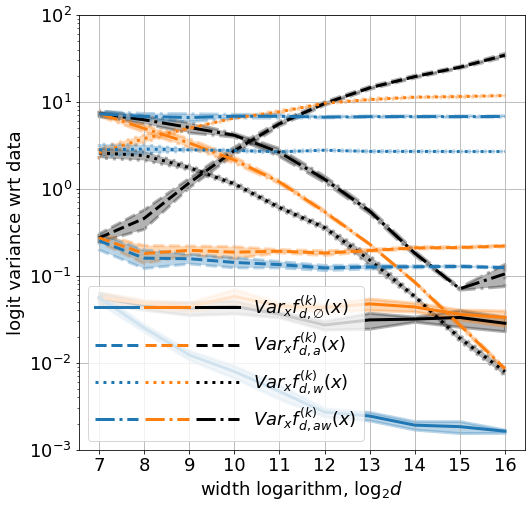}
    \includegraphics[width=0.29\textwidth]{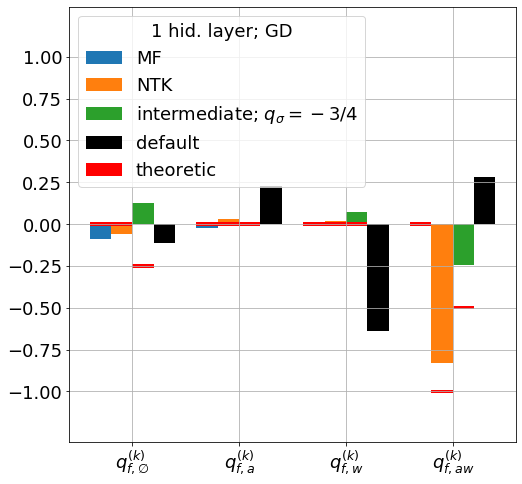}
    \caption{\textbf{Considering a multi-class classification instead of a binary one does not introduce any qualitative changes.} \textit{Top row:} scaling a reference network trained with a full-batch GD with (unscaled) learning rates $\eta_a^* = \eta_w^* = 0.02$ for 50 gradient steps on a subset of CIFAR2 (a dataset of first two classes of CIFAR10) of size 1000. \textit{Middle row:} same for a subset of MNIST of size 1000. \textit{Bottom row:} scaling a reference network trained with SGD using batches of size 100 with (unscaled) learning rates $\eta_a^* = \eta_w^* = 0.02$ for 6000 gradient steps on MNIST dataset. \textit{Left:} a final test cross-entropy (CE) loss as a function of width $d$. \textit{Center:} variance with respect to the data distribution for terms of model decomposition (\ref{eq:f_decomposition_1hid_appendix}) as a function of width $d$. \textit{Right:} numerical estimates for exponents of decomposition (\ref{eq:f_decomposition_1hid_appendix}) terms, as well as their theoretical values (denoted by red ticks). See SM~\ref{sec:appendix_experiments} for further details.}
    \label{fig:mnist_1hid_sgd_test_loss_and_var_f}
\end{figure*}

\end{document}